\documentclass{article}

\usepackage{amsmath}
\usepackage{amsthm}
\usepackage{dsfont}
\usepackage{amssymb}
\usepackage{bm}
\usepackage[round]{natbib}
\usepackage{booktabs}
\usepackage{hyperref}
\hypersetup{
    colorlinks=True,
    linkcolor=blue,
    filecolor=blue,      
    urlcolor=blue,
    citecolor=blue
}
\usepackage{xcolor}
\usepackage{tcolorbox}
\tcbset{boxsep=0mm,boxrule=0pt,colframe=white,arc=0mm,left=0.5mm,right=0.5mm}
\usepackage{graphicx}
\usepackage{microtype}
\usepackage{graphicx}
\usepackage{subfigure}
 \usepackage{amsfonts}
\usepackage{dsfont}
\usepackage{booktabs} 
\usepackage{bm}
\usepackage{appendix}
\usepackage{hyperref}

\newtheorem{theorem}{Theorem}
\newtheorem{corollary}{Corollary}[theorem]
\newtheorem{lemma}[theorem]{Lemma}
\newtheorem{definition}{Definition}

\newtheorem*{theorem*}{Theorem}
\newtheorem*{corollary*}{Corollary}
\newtheorem*{lemma*}{Lemma}
\usepackage[accepted]{icml2021}

\icmltitlerunning{Uniform Convergence, Adversarial Spheres and a Simple Remedy}

\begin{document}

\twocolumn[
\icmltitle{Uniform Convergence, Adversarial Spheres and a Simple Remedy}




\begin{icmlauthorlist}
\icmlauthor{Gregor Bachmann}{to}
\icmlauthor{Seyed-Mohsen Moosavi-Dezfooli}{to}
\icmlauthor{Thomas Hofmann}{to}
\end{icmlauthorlist}

\icmlaffiliation{to}{Department of Computer Science, ETH Z\"urich}
\icmlcorrespondingauthor{Gregor Bachmann}{gregor.bachmann@inf.ethz.ch}

\icmlkeywords{Machine Learning, ICML}

\vskip 0.3in
]



\printAffiliationsAndNotice{} 

\begin{abstract}
Previous work has cast doubt on the general framework of uniform convergence and its ability to explain generalization in neural networks. By considering a specific dataset, it was observed that a neural network completely misclassifies a projection of the training data (adversarial set), rendering any existing generalization bound based on uniform convergence vacuous. We provide an extensive theoretical investigation of the
previously studied data setting through the lens of infinitely-wide models. We prove that the Neural Tangent Kernel (NTK) also suffers from the same phenomenon and we uncover its origin. We highlight the important role of the output bias and show theoretically as well as empirically how a sensible choice completely mitigates the problem. We identify sharp phase transitions in the accuracy on the adversarial set and study its dependency on the training sample size. As a result, we are able to characterize critical sample sizes beyond which the effect disappears. Moreover, we study decompositions of a neural network into a clean and noisy part by considering its canonical decomposition into its different eigenfunctions and show empirically that for too small bias the adversarial phenomenon still persists.
\end{abstract}

\section{Introduction}
\label{submission}

Neural networks have achieved astonishing performance across many learning tasks such as in computer vision \citep{he2015deep}, natural language processing \citep{devlin2019bert} and graph learning \citep{kipf2017semisupervised}. The theoretical understanding of the generalization capability of these models, on the other hand, has been lagging behind in development and can so far only offer limited insights into the inner workings of these algorithms. Almost every work concerning generalization is based on the paradigm of uniform convergence as a tool to 
bound the capacity of the model \cite{arora2018stronger, bartlett2017nearlytight, bartlett2017spectrallynormalized, neyshabur2015normbased, neyshabur2018pacbayesian}. Recently however, \citet{nagarajan2019uniform} have cast doubt on the power of this technique. By constructing a dataset consisting of two concentric spheres (referred to as adversarial spheres), they were able to show that a neural network misclassifies a specific projection of the training data entirely. The existence of such an adversarial dataset 
renders any generalization bound based on uniform convergence vacuous. This surprising behaviour has been shown to hold empirically but, to the best of our knowledge, neither a mathematical proof nor a theoretical account of its origin has been given in the literature.  \\[3mm]
In this work, we revisit the aforementioned dataset and study the phenomenon of the model mathematically through the lens of infinitely wide neural networks. We leverage the analytic structure of the Neural Tangent Kernel (NTK) \cite{jacot2020neural} to prove the observed behaviour as well as unravel the dependencies on different parameters such as the sample size and the magnitude of 
the bias of the output layer of the model. These theoretical findings suggest a very simple fix consisting in increasing the output bias sufficiently. We validate our theoretical results using numerical experiments on the adversarial spheres dataset. 
Moreover, we explore the hypothesis put forth by \citet{nagarajan2019uniform} suggesting that there may exist a decomposition of the model into a clean and a noisy part. 
The noisy submodel should encapsulate the observed degeneracies while the clean submodel enjoys good generalization and robustness, making it amenable to uniform convergence. We investigate the most natural decomposition induced by the eigendecomposition of the kernel and show that even a restriction to the optimal set of eigenfunctions 
does not eliminate the adversarial effect. \\[3mm]
Our mathematical analysis suggests that the failure of uniform convergence in this particular setting is not pointing towards a deeper problem in neural architectures but is rather a result of the specific dataset and the architectural bias encouraging the network to rely on angular features instead of radial information. This questions the relevance of the observation in \citet{nagarajan2019uniform} regarding more realistic datasets containing angular structure.

We structure our work as follows. We first discuss related work in Section \ref{related}, followed by an overview of the mathematical setting and notation in Section \ref{notation}. In Section \ref{prevres}, we proceed to summarize the main results of \citet{nagarajan2019uniform} and \citet{jacot2020neural} as we build upon their findings. We then present our own theoretical and numerical results in Sections \ref{ours1} and \ref{ours2}, detailing the origin of the adversarial effect and its behaviour under a decomposition of the model. Finally, we provide a discussion of the implications of our work in Section \ref{discussion}.
\section{Related Work}
\label{related}
The goal of understanding generalization capabilities of neural networks gave rise to a rich line of work. A multitude of approaches to this task have been explored in the literature, \citet{bartlett2017spectrallynormalized, neyshabur2015normbased} for instance derive guarantees based on Rademacher complexities and covering numbers, resulting in upper bounds involving diverse norms of the weight matrices of the network. Other works investigate how compressing the model might help to derive meaningful guarantees, ensuring that the original and the compressed version remain close \cite{arora2018stronger,  zhou2019nonvacuous}. Others focus on randomized neural networks, leveraging the rich PAC-Bayesian theory to derive non-vacuous bounds \cite{dziugaite2017computing, zhou2019nonvacuous}. Derandomization of those bounds on the other hand strongly deteriorates their effectiveness \cite{neyshabur2018pacbayesian, nagarajan2019deterministic}.\\ The underlying framework shared between these diverse approaches is uniform convergence. This widely used paradigm has been recently questioned by \citet{nagarajan2019uniform}, demonstrating its failure in the most optimistic setting for a neural network with a very simple data distribution. To the best of our knowledge, little to no work in the literature has provided a theoretical account of this phenomenon or described its origin mathematically. The work closest to ours is \citet{negrea2020defense}, describing how a (possibly random) surrogate of the model can make uniform convergence applicable again. We however directly analyze the model in question instead of studying an approximation. Thus any insights derived from our analysis can point to a deeper problem in neural architectures. \\
Similar limitations have been discovered for kernel regression \citep{belkin2018understand} but in contrast to \citet{nagarajan2019uniform}, the results only apply in the presence of label noise. \\[3mm]
Recent works have established a direct correspondence between kernel regression and an infinitely wide fully-connected neural network at initialization \cite{lee2018deep} as well as during gradient flow training \cite{jacot2020neural}. Various follow-up works have refined these results, extending the analysis to various architectures \cite{arora2019exact, huang2020deep, du2019graph} and discrete gradient descent \cite{lee2019wide}.
The direct connection to the field of kernel regression makes the mathematical analysis of various phenomena in neural network training tractable. We leverage the convenient closed-form expression for a network trained with gradient descent in order to unravel the degeneracy of the model on the adversarial dataset outlined in \citet{nagarajan2019uniform}. 
\section{Notation and Definitions}
\label{notation}
We will establish some notation for the quantities of interest throughout this paper. Denote a fully-connected $L$-layer neural network through the recursive equations
\begin{itemize}
    \item $\bm{f}^{(l+1)}(\bm{x}) = \bm{W}^{(l+1)}\bm{\alpha}^{(l)}(\bm{x}) + \bm{b}^{(l+1)}$
    \item $\bm{\alpha}^{(l)}(\bm{x}) = \sigma\left(\bm{f}^{(l)}(\bm{x})\right)$
\end{itemize}
where $l=0, \dots L-1$, $f^{(0)}(\bm{x}) = \bm{x}$ and scalar output $f^{(L)}(\bm{x}) = \bm{W}^{(L)}\bm{\alpha}^{(L-1)}(\bm{x}) + b^{(L)} \in \mathbb{R}$. We have an input $\bm{x} \in \mathcal{X} \subset \mathbb{R}^{d}$, weight matrices $\bm{W}^{(l)} \in \mathbb{R}^{d_l \times d_{l-1}}$, biases $\bm{b}^{(l)} \in \mathbb{R}^{d_l}$ and a component-wise non-linearity $\sigma:\mathbb{R} \xrightarrow[]{} \mathbb{R}$. We denote by $\mathcal{F}$ the function class consisting of all possible neural networks. Moreover, define $\bm{\theta} \in \mathbb{R}^{M}$ as the concatenation of all parameters $\left(\bm{W}^{(1)}, \bm{b}^{(1)}\right), \dots, \left(\bm{W}^{(L)}, b^{(L)}\right)$ of the network, where $M \in \mathbb{N}$ denotes the total number of parameters in the model. 
\\[2mm] Consider a dataset $\mathcal{S} = \{(\bm{x}_1, y_1), \dots, (\bm{x}_n, y_n)\}$ where $(\bm{x}_i, y_i) \stackrel{\text{i.i.d}}{\sim} \mathcal{D}$  
for $i=1, \dots, n$ are distributed according to some probability distribution $\mathcal{D}$. We refer to $\bm{x}_i \in \mathcal{X}$ as the input with corresponding targets $y_i \in \mathcal{Y} \subset \mathbb{R}$. To make the notation clearer, we will sometimes use $y_{\bm{x}}$ to denote the label $y$ corresponding to $\bm{x}$. Occasionally, we will use $\bm{x}_i \sim p$ where $p$ is the marginal distribution of $\mathcal{D}$ with respect to the inputs.  We will denote by $\bm{X}\in \mathbb{R}^{n\times d}$ and $\bm{y} \in \mathbb{R}^{n}$ the stacking of all observations into a matrix and vector respectively. 
We define a loss function $L_f:\mathcal{X}\times \mathcal{Y} \xrightarrow[]{}\mathbb{R}$ that quantifies how close the prediction $f(\bm{x}_i)$ is to the ground truth $y_i$. We then train the model to minimize the empirical loss $\hat{L}$ consisting of the losses incurred on each sample:
$$\hat{L}_S: \mathcal{F} \xrightarrow{} \mathbb{R} \text{ , } f \mapsto \hat{L}_S(f) = \sum_{i=1}^{n}L_f(\bm{x}_i, y_i)$$
The more important quantity from a practical point of view, however, is given by the generalization error of the model:
$$L: \mathcal{F} \xrightarrow[]{}\mathbb{R}\text{, } f \mapsto L(f) = \mathbb{E}_{(\bm{x}, y) \sim \mathcal{D}}[L_f(\bm{x}, y)]$$

Understanding how much the generalization error can deviate from the empirical loss for a given data distribution and model is of paramount importance both in theory as well as in practice.
\section{Uniform Convergence and Neural Tangent Kernel}
\label{prevres}
In this section we give a brief overview of the previous work we will build upon. We first outline the key result of \citet{nagarajan2019uniform} on uniform convergence in order to motivate our theoretical analysis. We then shortly summarize the NTK framework introduced in \citet{jacot2020neural} as it serves as the main tool in our work. 
\subsection{Uniform Convergence and its Weaknesses}
\label{unifweak}
 Recently, \citet{nagarajan2019uniform} investigated how well the performance of neural networks can be captured by the very general machinery of uniform convergence. They study the most optimistic setup for uniform convergence by assuming that a perfect characterization of the solution space of gradient descent is known, critically reducing the hypothesis space needed to control. Under this assumption, a dataset is constructed which provably cannot be explained by uniform convergence. The argument goes along the following lines. Assume we have some algorithm $\mathcal{A}$ that chooses $f \in \mathcal{F}$ given a particular realization of a training set $\mathcal{S} = \{(\bm{x}_i, y_i)\}_{i=1}^{n}$. A uniform convergence bound is defined as the smallest $\epsilon_{\text{unif}}$ such that:
$$\mathbb{P}_{\mathcal{S} \sim \mathcal{D}^{n}}\left(\sup_{f \in \mathcal{F}}|L(f)-\hat{L}_S(f)| \leq \epsilon_{\text{unif}}\right)\geq 1-\delta$$
One considers a supremum over $\mathcal{F}$ to strip the chosen hypothesis $f \in \mathcal{F}$ of its complicated dependency on the data, making it more amenable to mathematical analysis. However $\mathcal{A}$ will never pick most of the hypotheses in $\mathcal{F}$, leading to an inflation of $\epsilon_{\text{unif}}$. Ideally, to reduce the supremum, one would restrict $\mathcal{F}$ to only those hypotheses that are considered by $\mathcal{A}$, denoted by $\mathcal{F}_{\mathcal{A}}$. Further pruning the search space is not possible as we would exclude models that could actually be chosen by $\mathcal{A}$. We can reformulate the uniform bound as follows. Consider a set of sets $\mathcal{S}_{\delta}$ consisting of different realizations of training sets $\mathcal{S}$ such that
$$\mathbb{P}_{\mathcal{S} \sim \mathcal{D}^n}(\mathcal{S} \in \mathcal{S}_{\delta}) \geq 1-\delta$$
Then the most optimistic uniform bound is given by the smallest $\epsilon_{\text{unif}, \mathcal{A}}$ such that
$$\sup_{S \in \mathcal{S}_{\delta}}\sup_{f \in \mathcal{F}_{\mathcal{A}}} |L(f)-\hat{L}_S(f)| \leq \epsilon_{\text{unif}, \mathcal{A}}$$ 
This formulation reveals the following weakness: Even if a classifier generalizes well ($L(f)=0$), we might still be able to leverage the data dependence of $f \in \mathcal{F}_{\mathcal{A}}$ on a particular draw $\mathcal{S} \in \mathcal{S}_{\delta}$ to construct a new training set $\mathcal{S}' \in \mathcal{S}_{\delta}$ as a function of $\mathcal{S}$ for which $\hat{L}_{\mathcal{S}'}(f)$ is big. \citet{nagarajan2019uniform} construct such an in-distribution adversarial construction $\mathcal{S}'$ for a very simple data distributions and every dataset $\mathcal{S} \in \mathcal{S}_{\delta}$, resulting in a huge supremum and thus provably vacuous generalization bounds. We will describe said construction in the following.

\subsection{Adversarial Spheres}
\label{advspheres}
\begin{figure}
    \centering
    \includegraphics[width=0.4\textwidth]{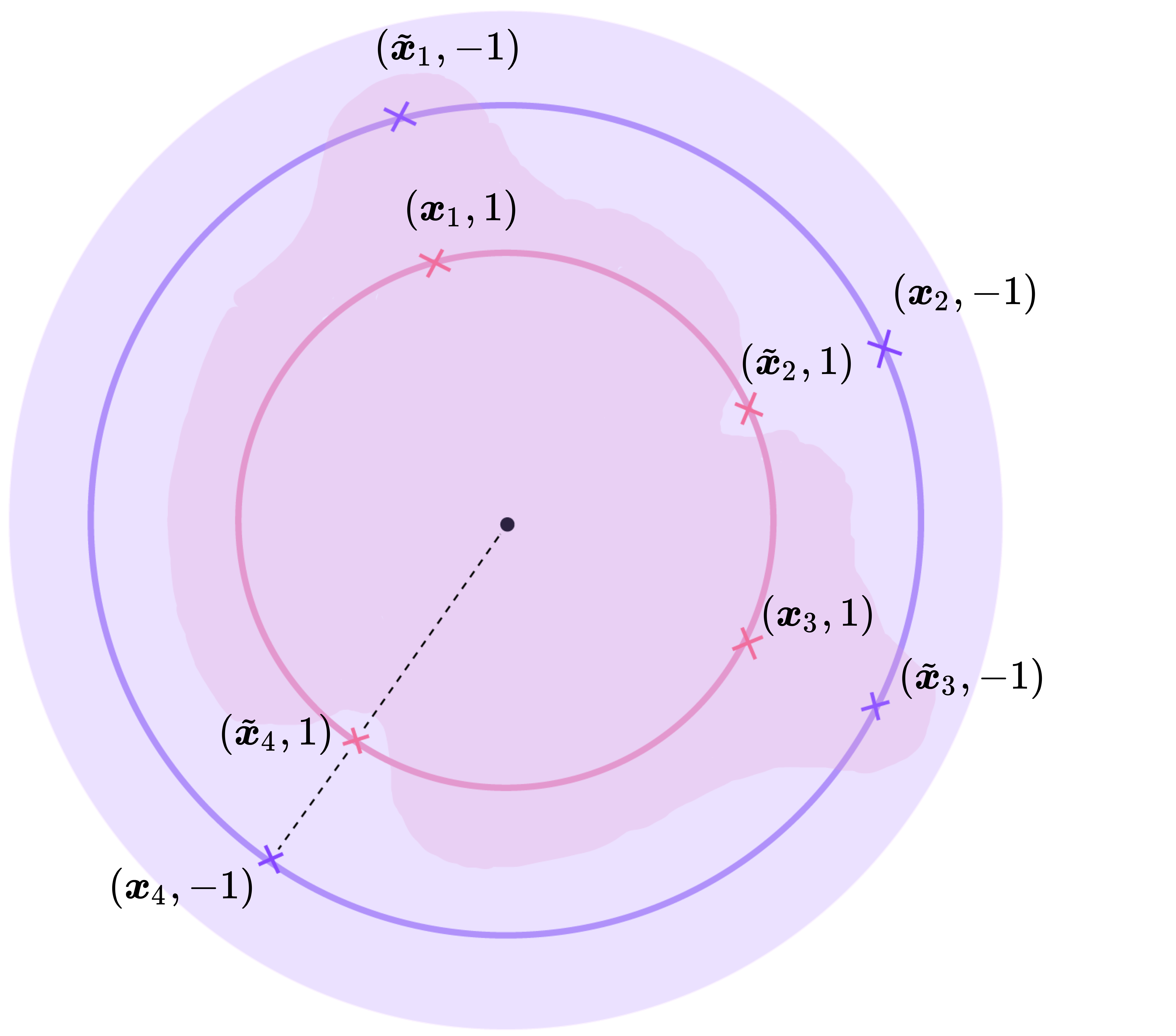}
    \caption{Schematic visualization of the decision boundary induced by a neural network trained with gradient descent on the $2$ dimensional adversarial spheres. Here the training set is $\mathcal{S}_{\text{train}} = \{\left(\bm{x}_1, 1\right), \left(\bm{x}_2, -1\right), \left(\bm{x}_3, 1\right), \left(\bm{x}_4, -1\right)\}$ and adversarial set is $\mathcal{S}_{\text{adv}} = \{\left(\tilde{\bm{x}}_1, -1\right), \left(\tilde{\bm{x}}_2, 1\right), \left(\tilde{\bm{x}}_3, -1\right), \left(\tilde{\bm{x}}_4, 1\right)\}$. Notice how the data distribution is captured correctly except for $\mathcal{S}_{\text{adv}}$ which is misclassified completely.}
    \label{advspheresEx}
\end{figure}
Consider the following simple dataset described by the input data distribution
$$ p(\bm{x}) = qp_{r_1}(\bm{x}) + (1-q)p_{r_2}(\bm{x})$$
with $r_1 < r_2$, $0<q<1$, and $p_{r}(\cdot)$ the uniform density over the sphere $\bm{S}_{r}^{d-1}=\{\bm{x} \in \mathbb{R}^{d}: ||\bm{x}||_2 = r\}$: 
$$p_r(\bm{x}) = \frac{1}{A_r}\mathds{1}_{\{\bm{x} \in \bm{S}^{d}_{r}\}}$$
where $A_r$ denotes the surface area of a $d-1$-dimensional sphere with radius $r$. Whenever $||\bm{x}||_2 = r_1$, we label the point as $y_{\bm{x}} = 1$ and when $||\bm{x}||_2 = r_2$ we set $y_{\bm{x}} = -1$. We define the class probabilities as $\mathbb{P}(y=1) = q$ and $\mathbb{P}(y=-1)=1-q$. As before, we refer to $p$ as the input distribution ($\bm{x} \sim p)$ and to $\mathcal{D}$ as the data distribution ($(\bm{x}, y) \sim \mathcal{D})$.\\[3mm]
It will be very important to study how a point $\bm{x} \sim p$ will determine the behaviour of a model $f$ at the corresponding projection $\tilde{\bm{x}}$ on the other sphere. To this end we introduce the projection 
$$\bm{x} \mapsto \mathcal{P}(\bm{x}) = \frac{r_1}{r_2}\bm{x} \mathds{1}_{\{\bm{x} \in \mathbb{S}^{d}_{r_2}\}} + \frac{r_2}{r_1}\bm{x} \mathds{1}_{\{\bm{x} \in \mathbb{S}^{d}_{r_1}\}}$$
We will refer to $\mathcal{P}(\bm{x})$ both as the projection of $\bm{x}$ as well as the adversarial point of $\bm{x}$. We call the set $$\mathcal{S}_{\text{adv}}=\big{\{}\left(\mathcal{P}(\bm{x}_i), -y_i\right): i=1, \dots, n\big{\}}$$
the adversarial set. Crucially, the distribution of $(\bm{x}, y)$ remains invariant under $\mathcal{P}$ due to the uniformity on both spheres. As empirically observed in \citet{nagarajan2019uniform}, surprisingly, a neural network trained by gradient descent on $\mathcal{S}$ completely misclassifies $\mathcal{S}_{\text{adv}}$. We depict this phenomena for the $2$-dimensional case in Figure \ref{advspheresEx}. As a consequence, any uniform convergence-based bound is rendered vacuous, as outlined in Section \ref{unifweak}.\\ This observation questions the validity of uniform convergence as it already fails to explain the generalization on such a simple data distribution for quite generic neural networks.\\ It is thus crucial to understand how this degeneracy in neural networks arises and to determine if the effect is simply a consequence of the particular data distribution or pointing to a deeper problem of neural architectures.
\subsection{NNGP and Neural Tangent Kernel}
\label{ntk}
Recently, a novel tool for analyzing neural networks emerged in the form of the NNGP \cite{lee2018deep} and the NTK \cite{jacot2020neural}. These works assume a different parametrization of the network by introducing a scaling $\frac{1}{\sqrt{d_l}}$ at each layer.
Every weight is initialized according to $W^{(l)}_{ij} \stackrel{i.i.d.}{\sim} \mathcal{N}(0, 1)$ whereas the bias follows $b_i^{(l)} \sim \mathcal{N}(0, \beta_l^2)$. As shown in \citet{lee2018deep}, as the widths $d_i \xrightarrow[]{} \infty$ for $i=1, \dots, L$, the neural network at initialization exhibits a Gaussian process behaviour:
$$f(\cdot) \sim \mathcal{GP}({0}, {\Sigma}^{(L)})$$
governed by the NNGP kernel $\Sigma^{(L)}$ defined recursively as $\Sigma^{(1)}(\bm{x}, \bm{x'}) = \frac{1}{\sqrt{d_0}}\bm{x}^{T}\bm{x'} + \beta_l^2$ and $$\Sigma^{(l)}(\bm{x}, \bm{x'}) =  \mathbb{E}_{\bm{z} \sim \mathcal{N}(\bm{0},\tilde{\bm{\Sigma}}^{(l-1)})}\big[\sigma(z_1)\sigma(z_2)\big] + \beta_l^2$$
for $l=1, \dots, L$ and where $\tilde{\bm{\Sigma}}^{l} \in \mathbb{R}^{2 \times 2}$ is obtained from evaluating $\Sigma^{l}$ on the set ${\{\bm{x}, \bm{x’}\}}$.
\citet{jacot2020neural} extended this result by incorporating gradient descent dynamics. They introduced the empirical neural tangent kernel
$$\hat{\Theta}(\bm{x},\bm{x}') = \left(\nabla_{\bm{\theta}}f(\bm{x})\right)^{T}\nabla_{\bm{\theta}}f(\bm{x}')$$
and showed that in the infinite-width regime, the kernel becomes deterministic and remains constant along the training trajectory induced by gradient flow. Moreover, the limiting kernel, denoted by $\Theta$, has a closed-form expression given by the recursion
$\Theta^{(1)}(\bm{x}, \bm{x}') = \Sigma^{(1)}(\bm{x}, \bm{x}')$ and
    $$\Theta^{(L+1)}(\bm{x}, \bm{x}') = \Theta^{(L)}(\bm{x}, \bm{x}') \dot{\Sigma}^{(L+1)}(\bm{x}, \bm{x}') + \Sigma^{(L+1)}(\bm{x}, \bm{x}')$$

where $\dot{\Sigma}^{(l)}(\bm{x}, \bm{x}') = \mathbb{E}_{\bm{z} \sim \mathcal{N}(\bm{0},\tilde{\bm{\Sigma}}^{(l-1)})}\big[\dot{\sigma}(z_1)\dot{\sigma}(z_2)\big]$. As a consequence, under mean squared loss, neural network training can be viewed as kernel regression and admits a simple formula for a network trained for infinitely long:
$$f_{\infty}(\bm{x}) = \left(\Theta^{(L)}(\bm{x}, \bm{X})\right)^{T} \left(\Theta^{(L)}(\bm{X}, \bm{X})\right)^{-1}\bm{y}$$
This convenient formula lends itself better to mathematical analysis, compared to finite-width networks, while still preserving lots of important structure. \\
\citet{lee2019wide} later extended the analysis to gradient descent with a small enough step size and proved that the NNGP can be viewed as the limiting kernel of a neural network with only a trainable last layer. Thus the NNGP can be viewed as a special case of the NTK.
\section{Adversarial Spheres and Infinite Width}
\label{ours1}
We now turn to the study of adversarial spheres in the infinite width setting, employing both the NNGP and the NTK. Although clearly a classification task, we will use mean squared loss as often done in the literature \cite{chen2020labelaware, arora2019exact}.  
For generality and elegance of the argument, we define a general class of kernels that admit a certain property. \begin{definition} Consider a kernel $K: \mathbb{R}^{d} \times \mathbb{R}^{d}$. We call $K$ \textbf{semi-homogeneous} if and only if there exists $\zeta \in \mathbb{R}$ such that $\forall \bm{x}, \bm{x}' \in \mathbb{R}^{d}$ and $\alpha>0$, it holds that
$$K(\alpha \bm{x}, \bm{x}') = \alpha K(\bm{x}, \bm{x}') + \zeta^2(1-\alpha)$$
\end{definition}
We have the following theorem that shows that both the NTK and the NNGP using ReLU non-linearity belong to the family of semi-homogeneous kernels. Due to its strongly technical nature, we postpone the proof to the appendix.
\begin{theorem}
\label{ntksemi}
Consider a fully-connected neural network with NTK parametrization as introduced in Section \ref{ntk}, equipped with a $1$-homogeneous activation function (such as ReLU). Set every bias to zero ($\beta_i = 0$) except for the output bias, $b^{(L)} \sim \mathcal{N}(0, \beta^2)$. Then it holds that both $\Theta^{(L)}$ and $\Sigma^{(L)}$ are semi-homogeneous kernels with $\zeta = \beta$.
\end{theorem}
Semi-homogeneous kernels form an interesting family of kernels for the adversarial spheres because one can easily quantify their change under the projection operator $\mathcal{P}$ introduced in Section \ref{advspheres}, as shown in the following lemma.
\begin{lemma}
\label{semikernel}
Fix a semi-homogeneous kernel $K$ and two data points sampled according to the adversarial spheres measure, $\bm{x}, \bm{z} \sim p$. Consider the projection $\mathcal{P}(\bm{x})$. Denote $r=||\bm{x}||_2$ and $\tilde{r} = ||\mathcal{P}(\bm{x})||_2$. Then it holds that:
 $$K(\mathcal{P}(\bm{x}), \bm{z}) = \frac{\tilde{r}}{r}K(\bm{x}, \bm{z}) + \zeta^2(1-\frac{\tilde{r}}{r})$$
\end{lemma}

This offers the following interesting insight. A semi-homogeneous kernel is only affected under the projection $\mathcal{P}$ through the magnitude of the inputs. In the case of the adversarial spheres, 
the change is thus entirely determined through the label information $y_{\bm{x}}$ since it is a function of $||\bm{x}||_2$. The angular component in $\bm{x}$ is completely irrelevant to the model. This becomes more crucial when studying the predictive function $f_K$ induced by kernel regression with $K$ under mean squared loss:
$$f_K(\bm{x}) = K(\bm{x}, \bm{X}) K(\bm{X}, \bm{X})^{-1}\bm{y}$$
Using the insight from the previous lemma, we can relate the prediction of $f_K$ on both $\bm{x}$ and $\mathcal{P}(\bm{x})$, which is a crucial step towards understanding how the performance of $f_K$ on $\mathcal{S}_{\text{train}}$ relates to the one on $\mathcal{S}_{\text{adv}}$.

\begin{corollary}
\label{pred_cor}
Fix a semi-homogeneous kernel $K$ and a data point sampled according to the adversarial spheres measure, $\bm{x} \sim p$. Consider the projection $\mathcal{P}(\bm{x})$. Denote $r=||\bm{x}||_2$ and $\tilde{r} = ||\mathcal{P}(\bm{x})||_2$. Then it holds that 
$$f_K\left(\mathcal{P}(\bm{x})\right) = \frac{\tilde{r}}{r}f_K(\bm{x}) + \zeta^{2}\left(1- \frac{\tilde{r}}{r}\right)\gamma_K(n)$$
where we define $\gamma_K(n) = \bm{1}_n^{T}K(\bm{X}, \bm{X})^{-1}\bm{y}$ and $\bm{1}_n = \left(1, \dots, 1\right)^{T} \in \mathbb{R}^{n}$.
\end{corollary}
Crucially, $\gamma_K(n)$ is entirely agnostic to the data point $(\bm{x}, y)$ and solely depends on the kernel $K$ and the training data $\mathcal{S}_{\text{train}}$. Once more, the label information fully determines how $f_K$ will change under the projection. \\[3mm]
\subsection{Adversarial Accuracy}


Equipped with these results we can now turn our attention to the adversarial set $\mathcal{S}_{\text{adv}}$ and the resulting accuracy,
$$a_{\text{adv}} = \frac{1}{n}\sum_{i=1}^{n} \mathds{1}_{\big{\{}\operatorname{sgn}\left(f_K(\mathcal{P}\left(\bm{x}_i\right)\right) = -y_i\big{\}}}$$
coined adversarial accuracy and present our main theoretical findings. In \citet{nagarajan2019uniform}, it was empirically observed that $$a_{\text{adv}} = 0$$
for a specific neural architecture. Here we are able to prove this phenomenon mathematically for semi-homogeneous kernels and unravel the dependencies of $a_{\text{adv}}$ on parameters like the sample size $n$, the radii $r_1, r_2$ and the semi-homogeneous parameter $\zeta$.
\begin{theorem}
\label{advacc}
Take a semi-homogeneous kernel K and consider a training set $\mathcal{S}_{\text{train}} \stackrel{\text{i.i.d.}}{\sim} \mathcal{D}^n$ along with the corresponding adversarial set $\mathcal{S}_{\text{adv}}$. Then it holds that $a_{\text{adv}}$ is quantized to only three values:
$$a_{\text{adv}} \in \Big{\{}0, 1-q, 1\Big{\}}$$
Moreover, we can characterize the phase transitions in sample size $n$ as 
$$a_{\text{adv}} = \begin{cases} 0 \hspace{8mm}\text{ if } \hspace{3mm} \gamma_K(n) \leq \frac{r_1}{\zeta^2(r_2-r_1)} \\[2mm]
1- q \hspace{2mm}\text{ if } \hspace{2mm} \frac{r_1}{\zeta^2(r_2-r_1)} \leq \gamma_K(n) \leq \frac{r_2}{\zeta^2(r_2-r_1)} \\[2mm]
1 \hspace{8mm}\text{ if }\hspace{3mm} \gamma_K(n) \geq \frac{r_2}{\zeta^2(r_2-r_1)} \end{cases}$$
\end{theorem}


We can see that the adversarial accuracy goes through phase transitions governed by $\gamma_K(n)$. We validate this surprising result through numerical experiments, displayed in Figure \ref{advaccplotNTK} and Figure \ref{advaccplotNNGP}. We use a two $100$-dimensional spheres with radii $r_1=1$ and $r_2=1.11$, similar to the setup considered in \citet{nagarajan2019uniform}. As predicted, we observe very sharp phase transitions in the adversarial accuracy and they occur exactly at the sample sizes predicted by our theory. Both NNGP and NTK indeed display the effect at small sample sizes but as $n$ increases we recover perfect accuracy. \\[3mm] To gain a better understanding in terms of the sample size $n$, we need to analyze $\gamma_K(n)$ in more detail.
\begin{figure}
    \centering
    \includegraphics[width=0.45\textwidth]{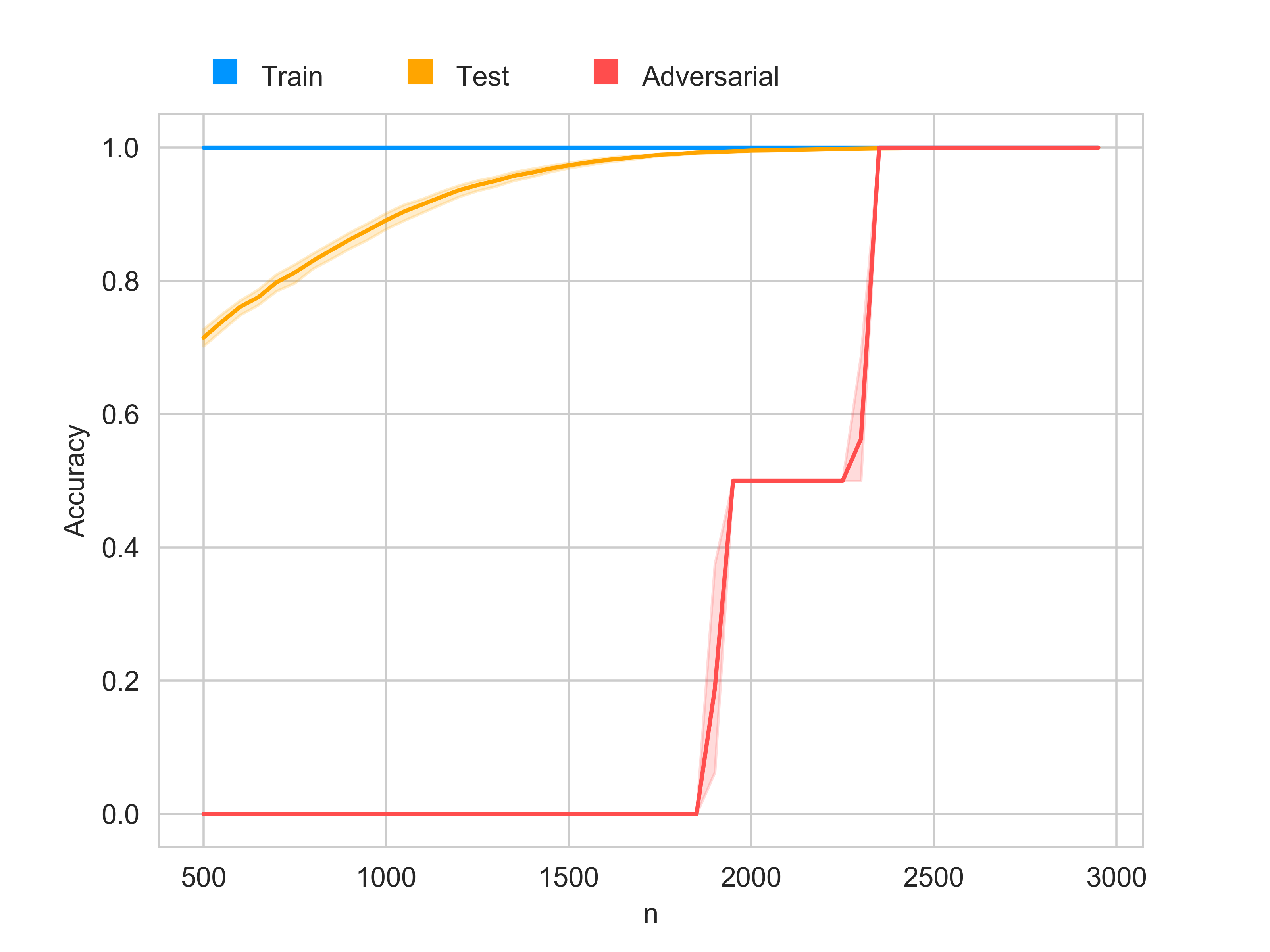}
    \vspace{-5mm}
    \caption{Train, test and adversarial accuracy of $3$ layer NTK evaluated on $100$-dimensional adversarial spheres, plotted against sample size $n$. Results are averaged over $8$ runs.}
    \label{advaccplotNTK}
\end{figure}

\subsection{Properties of $\gamma_K(n)$}
In this section, we restrict our attention to semi-homogeneous kernels of the form 
$$K(\bm{x}, \bm{x}') = C(\bm{x}, \bm{x}') + \beta^2$$
where $C:\mathbb{R}^{d} \times \mathbb{R}^{d}$ is a homogeneous kernel, $C(\alpha \bm{x}, \bm{x}') = \alpha C(\bm{x}, \bm{x}')$ $\forall \alpha >0$. A simple calculation indeed reveals that $K$ is semi-homogeneous. As outlined in the proof of Theorem \ref{ntksemi}, this restricted family still includes the NTK and the NNGP with an output bias $\beta$. As a first step, we can isolate the role of the semi-homogeneous parameter $\beta$.
\begin{lemma}
\label{isobias}
Assume that $K$ is of the above form and denote by $C$ the corresponding homogeneous kernel. Then it holds that 
$$\gamma_{K}(n) =  \frac{1}{1 + \beta^2{s}\left({C}(\bm{X},\bm{X})^{-1}\right)}\gamma_{{C}}(n)$$
where we define ${s}(\bm{A}) = \sum_{i,j}A_{ij}$.
\end{lemma}
Let us assume in the following that for simplicity, $q=\frac{1}{2}$  and that we have a balanced training dataset. Define $2m = n$ for $m \in \mathbb{N}$. Moreover, without loss of generality, we permute the order of the training samples such that the first $m$ entries in $\bm{y}$ correspond to the positive class ($y=1$) and the last $m$ entries to the negative class ($y=-1$). To get qualitative insights into $\gamma_K$, we analyze the behaviour in expectation over the dataset. With a slight abuse of notation, we define $\gamma_{\bm{A}}(n) = \bm{1}_n^{T}{\bm{A}}^{-1}\bm{y}$ for $\bm{A} \in \mathbb{R}^{n \times n}$.
\begin{figure}
    \centering
    \includegraphics[width=0.45\textwidth]{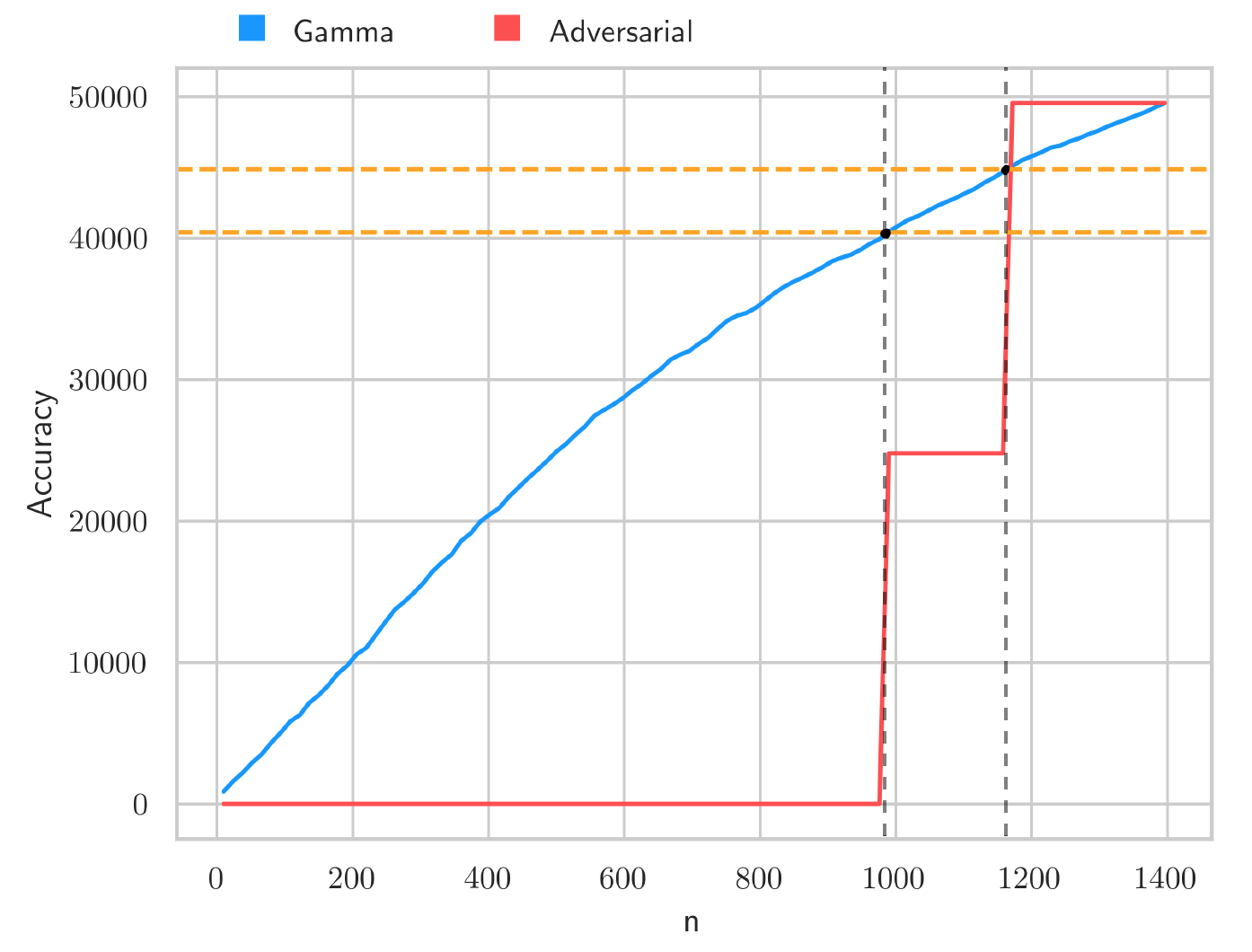}
    \vspace{-3mm}
    \caption{Scaled adversarial accuracy of $3$ layer NNGP evaluated on $100$-dimensional adversarial spheres, plotted against $\gamma_K(n)$. Horizontal lines indicate predicted phase transitions in $\gamma_K$.}
    \label{advaccplotNNGP}
\end{figure}
\begin{theorem}
\label{gamma}
Consider the expected kernel $\tilde{K} = \mathbb{E}_{\bm{X} \sim p^{n}}\left[K(\bm{X}, \bm{X})\right]$. We have that $\gamma_{\tilde{K}}$ is asymptotically given by  
$$\gamma_{\tilde{K}}(n) \propto  \frac{C_1 + \eta n}{C_2 - \beta^2C_3 n}$$
for constants $C_1, C_2, C_3, \eta \in \mathbb{R}$ and the limit is given by
$$\gamma_{\tilde{K}}(n) \xrightarrow[]{n \xrightarrow[]{}\infty}\frac{r_1+r_2}{\beta^2(r_2-r_1)}$$
\end{theorem}
Surprisingly, 
the limiting capacity is independent of the particular kernel except for its semi-homogeneous parameter $\beta$. 
Moreover, as intuitively expected, $\gamma_{\tilde{K}}(n)$ is an increasing function. As a consequence a model will experience the phase transitions outlined in Theorem \ref{advacc} in sequence. We verify our predictions numerically by plotting $\gamma_K$ for different kernels and  comparing them with the averaged case in Figure \ref{gamma1L}.  We can readily see that the kernel in expectation is a good approximation and provides a tight fit especially for moderately large to large $n$. We provide more numerical evidence in the Appendix \ref{gamma_behaviour}.
\begin{figure}
    \centering
    \includegraphics[width=0.45\textwidth]{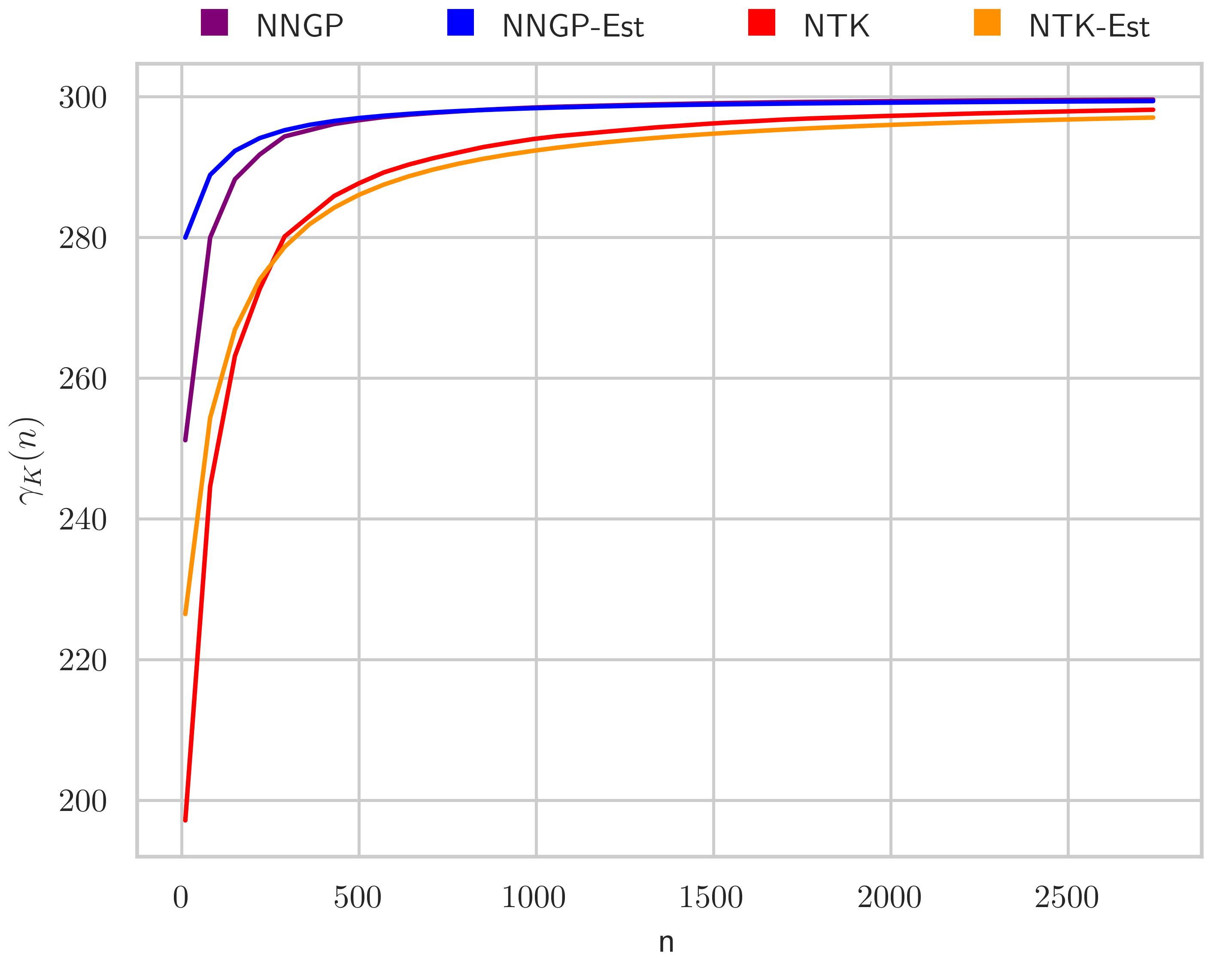}
    \vspace{-5mm}
    \caption{$\gamma_K(n)$ plotted against sample size $n$ for a $1$-layer NNGP and NTK along with the corresponding approximation $\gamma_{\tilde{K}}(n)$. Results are averaged over $8$ runs.}
    \label{gamma1L}
\end{figure}
\subsection{The Role of the Bias $\beta^2$}
In this section we will study the influence of the output bias $\beta$ for fixed sample sizes $n$. Again we restrict the analysis to kernels of the form $K(\bm{x}, \bm{x}') = C(\bm{x}, \bm{x}') + \beta^2$.
Due to Lemma \ref{isobias}, studying the behaviour of $a_{\text{adv}}$ for varying bias $\beta^2$ but fixed sample size $n$ now becomes feasible.
One can easily see that $$g(\beta) = \beta^2\gamma_K(n)=\frac{\beta^2}{1 + \beta^2{s}\left({C}(\bm{X},\bm{X})^{-1}\right)}\gamma_{{C}}(n)$$
is an increasing function in $\beta$. As a consequence, for a fixed sample size $n$, also $a_{\text{adv}}$ is increasing in $\beta$. We can calculate the capacity limit as 
$$\beta^2 \gamma_K(n) \xrightarrow[]{\beta \xrightarrow[]{} \infty} \frac{\gamma_{{C}}(n)}{s\left({C}\left(\bm{X}, \bm{X}\right)^{-1}\right)}$$
Thus an increasing bias leads to better robustness in terms of the adversarial accuracy but there is an upper limit to the benefit. Depending on this capacity limit, a big enough bias potentially leads to a perfect adversarial accuracy. As a result, a simple increase in the bias of the network could potentially mitigate the problem entirely. We verify our results again through numerical experiments. We fix the sample size $n \in \mathbb{N}$ such that for small bias $\beta$ we observe a strong adversarial effect. We then vary $\beta$ and show the test, train and adversarial accuracy as a function of $\beta$ in Figure \ref{BiasVsAcc}. Again we observe sharp phase transitions in the adversarial accuracy as well as an increase in generalization. Indeed, a bigger output bias alleviates the adversarial effect completely without any increase in sample size. Moreover, although our theory only holds for the infinite width case, we observe the same phenomenon for finite-width networks trained with gradient descent under mean squared error. In Figure \ref{finitewidth} we show the accuracies of a $2$ hidden layer network of width $1000$ plotted against different bias initialization magnitudes. Again we observe the same phase transitions in the adversarial accuracy.
\begin{figure}
    \centering
    \includegraphics[width=0.5\textwidth]{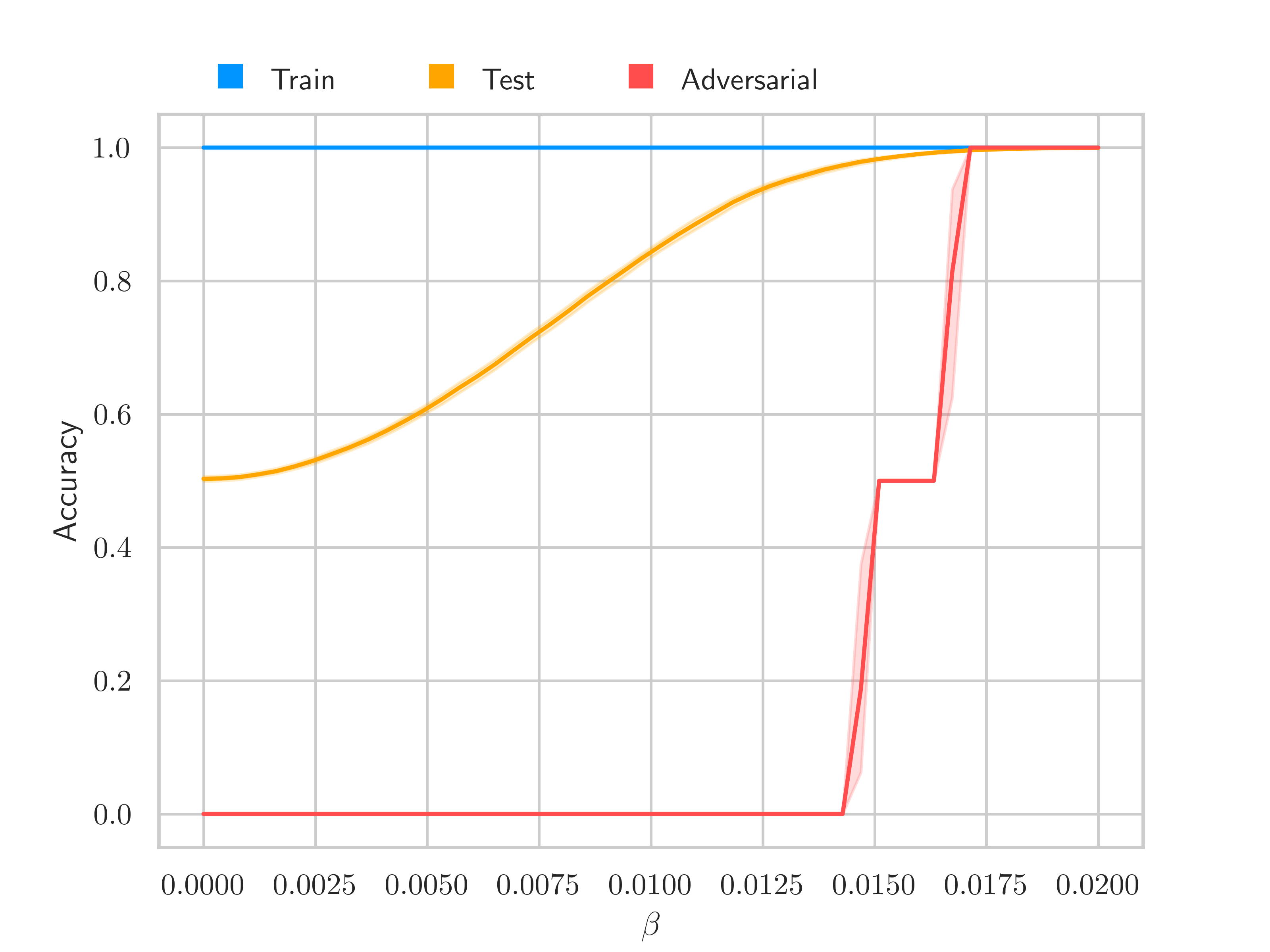}
    \vspace{-5mm}
    \caption{Different accuracies for a 2-layer NNGP model plotted against bias $\beta$ for fixed sample size. Results are averaged over $8$ runs.}
    \label{BiasVsAcc}
\end{figure}

\section{Decomposition of Neural Network}
\label{ours2}
We have identified sharp phase transitions in $a_{\text{adv}}$ and that a simple increase in the bias $\beta$ can completely mitigate the problem. In this section we explore an alternative approach, advocated by \citet{nagarajan2019uniform}, 
given by decompositions of the network into a clean part $f_{\text{clean}}$ and a noisy part $f_{\text{noisy}}$ such that
$$f(\bm{x}) = f_{\text{clean}}(\bm{x}) + f_{\text{noisy}}(\bm{x})$$
Ideally, $f_{\text{clean}}$ would capture the good generalization capability of $f$ while being more robust against the adversarial effect. Instead of analyzing $f$, one could study $f_{\text{clean}}$ with tools based on uniform convergence. Here we study the canonical decomposition of the network, induced by the eigenfunctions of the kernel $K$. We show empirically that such a decomposition does not alleviate the adversarial effect and that sufficient bias is still necessary. 

\subsection{Eigendecomposition of Kernel}
Consider the Mercer decomposition of a kernel $K$:
$$K(\bm{x}, \bm{x}') = \sum_{i=1}^{\infty}\lambda_i \phi_i(\bm{x}) \phi_i(\bm{x}')$$
where $\left(\phi, \lambda\right)$ is an eigenfunction-eigenvalue pair of the Fredholm integral operator
$$T_{K}^{p}: H \xrightarrow[]{} L^{2} \text{ , } \phi \mapsto \int_{\mathbb{R^{d}}} K(\bm{x},\bm{z}) \phi(\bm{z}) p(\bm{z})d\bm{z}$$
where $p$ denotes the input data measure and $H$ is some function space. The study of the eigenfunctions of $T_K$ for dot-product kernels has been mainly limited to the uniform measure over a single sphere \cite{basri2019convergence, bietti2019inductive}. Recently, \citet{basri2020frequency} have extended 

this analysis to a piece-wise constant density on the sphere. \\
In order to analyze eigendecompositions for the adversarial spheres, we need to understand the spectral properties of $T_K^p$. It turns out that for semi-homogeneous dot-product kernels $K$, one can extend the eigenanalysis to the more general class of isotropic distributions (see Theorem 5 in  \cite{geifman2020similarity}):
\begin{theorem}
Consider an input distribution $p(\bm{x})$ such that the conditional distribution $p(\bm{x}\big{|}||\bm{x}||_2 = r)$ is the uniform measure over $\bm{S}^{d-1}_r$. Denote by $p_R(r)$ the distribution of $||\bm{x}||_2$ and by $p_{1}$ the uniform measure over the sphere. Fix an eigenfunction eigenvalue pair $(\tilde{\phi}, \tilde{\lambda})$ of $T_K^{p_{1}}$.  Then we can express the eigenfunction eigenvalue pairs $\left(\phi, \lambda\right)$ of $T_K^{p}$ as
\vspace{-4mm}
\begin{itemize}
    \item $\tilde{\phi}(\bm{x}) = \frac{1}{\sqrt{\mathbb{E}_{R \sim p_R}\left[R^2\right]}}{\phi}(\bm{x})$
    \item $\tilde{\lambda} = \lambda \mathbb{E}_{R \sim p_R}\left[R^2\right] $
\end{itemize}
\end{theorem}

This theorem relates the eigenfunctions associated with the isotropic measure directly to the eigenfunctions of the rather well-understood uniform measure on the sphere. As shown for instance in \citet{basri2019convergence, bietti2019inductive}, the eigenfunctions of the NTK and NNGP 
are given by the spherical harmonics. The eigenvalues are trickier to study and depend on the structure of the employed kernel $K$. Some specific architectures such as one hidden layer networks do admit analytic expressions \cite{basri2019convergence}. \\[2mm]
In particular, for the adversarial spheres, we observe that
$$p_R(r) = q \delta_{r_1}(r) + (1-q)\delta_{r_2}(r)$$
\begin{figure}
    \centering
    \includegraphics[width=0.5\textwidth]{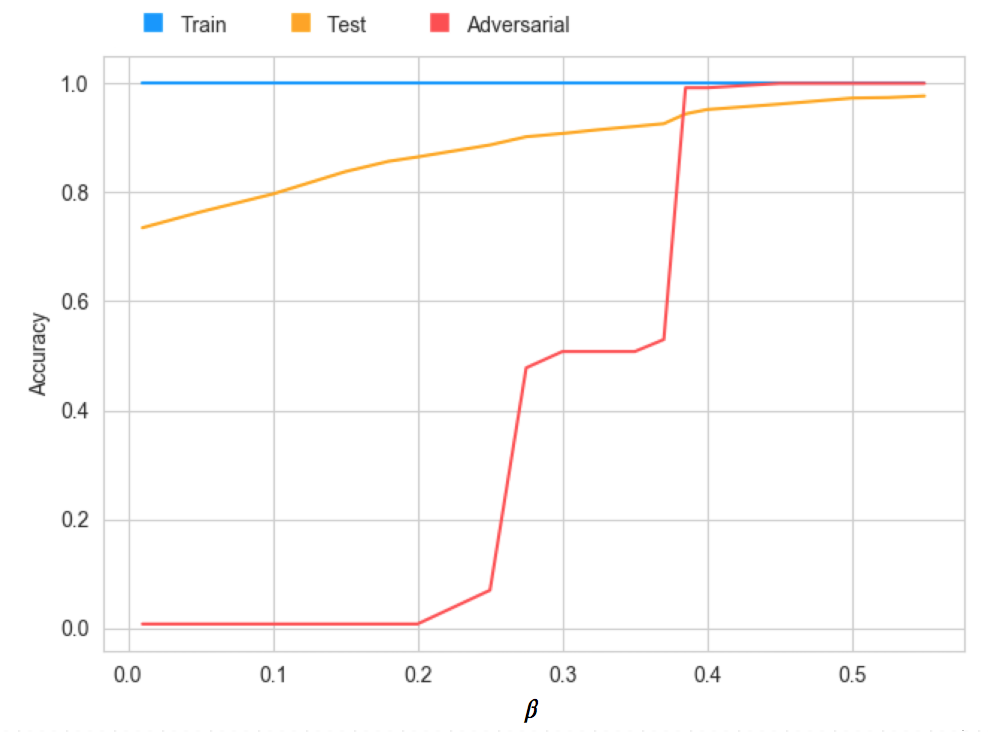}
    \vspace{-8mm}
    \caption{Different accuracies for a 2-layer neural network of width 1000 plotted against bias $\beta$ for fixed sample size, trained with gradient descent under MSE.}
    \label{finitewidth}
\end{figure}
\noindent where $\delta_{z}(x)$ denotes a Dirac Delta centered at z.
We can thus easily calculate $\mathbb{E}_{R \sim p_R}[R^2] = qr_1^2 + (1-q)r_2^2>0$ and hence conclude that the number of non-zero eigenvalues are the same as for the uniform measure over $\bm{S}^{d-1}_1$. 

\subsection{Canonical Decomposition of Predictive Function}
We want to investigate the question whether using a decomposition of the neural networks induced by the Mercer decomposition can alleviate the problem encountered in adversarial spheres.
The eigenfunctions associated with the integral operator however are not the ideal arena to study the problem as they are infinite sample quantities, stemming from the complete knowledge of the data distribution. The adversarial effect on the other hand is a finite sample effect that starts to vanish as the sample size $n$ increases, as seen in the previous section. As a result, we instead study finite-sample estimators of the eigenfunctions and eigenvalues. \\[3mm]
Consider the spectral decomposition of the kernel matrix
$$K(\bm{X}, \bm{X}) = \bm{V} \operatorname{diag}(\bm{\mu}) \bm{V}^{T} \in \mathbb{R}^{n \times n}$$
where $\bm{\mu} \in \mathbb{R}^{n}$ are the eigenvalues and $\bm{V} \in \mathbb{R}$ contains the associated eigenvectors. Using these quantities, we can form estimators of the eigenfunction $\phi$ and eigenvalues $\lambda$ as follows:
\vspace{-3mm}
\begin{itemize}
    \item $\hat{\lambda}_i = \frac{1}{n} \mu_i$
    \item $\hat{\phi}_i(\bm{x}) = \frac{1}{\mu_i}\sum_{k=1}^{n} V_{ki}K(\bm{x}_k, \bm{x})$
\end{itemize}
We refer to \citet{baker} for an in-depth treatment of these finite approximations to Fredholm integral problems.
These estimators in turn induce a decomposition on the predictive function at any finite sample size:
\begin{lemma}
Consider any kernel $K$ and its associated predictive function $f_K$. We can decompose $f_K$ into its different spectral components
$$f_K(\bm{x}) = \sum_{k=1}^{n}\left(\bm{v}_{k}^{T}\bm{y}\right) \hat{\phi}_k(\bm{x})$$
where $\bm{v}_k$ denotes the $k$-th eigenvector of $K(\bm{X}, \bm{X})$.
\end{lemma}
Essentially, $\bm{v}_i^{T}\bm{y}$ measures the importance of the eigenfunction $\hat{\phi}_i$ to the task. Eigenvectors that are well-aligned with the targets $\bm{y}$ will contribute more to the prediction while orthogonal eigenvectors will not be considered. This decomposition gives rise to very natural splittings of the form
$$f_{K}(\bm{x}) = \underbrace{\sum_{k \in \mathcal{I}}\left(\bm{v}_{k}^{T}\bm{y}\right) \hat{\phi}_k(\bm{x})}_{f_{\text{clean}}(\bm{x})} + \underbrace{\sum_{k \not \in \mathcal{I}}\left(\bm{v}_{k}^{T}\bm{y}\right) \hat{\phi}_k(\bm{x})}_{f_{\text{noisy}}(\bm{x})}$$
where $\mathcal{I} \subset \{1, \dots, n\}$ is an index set which can be varied. we will study numerically how restricting the full predictive function to such a subset of eigenfunctions might improve the adversarial accuracy, for a fixed small bias $\beta$. We refer to $\hat{\phi}_i$ with $i=\operatorname{argmax}_{1 \leq j \leq n}|\bm{v}_j^{T}\bm{y}|$ as the dominant eigenfunction. We study the decomposition
$$f_{\text{clean}}(\bm{x}) = \hat{\phi}_i(\bm{x})$$
Interestingly, $f_{\text{clean}}$ perfectly captures the data distribution, as illustrated in Figure \ref{cutacc}, visible in the perfect training and test accuracy. It however does not alleviate the adversarial effect completely as it persists for small sample sizes and only very slowly converges. We study different combinations of eigenfunctions in the Appendix \ref{more_eigen} but none can improve over the dominant eigenfunction in terms of adversarial accuracy. Again, only an increase in the output bias $\beta$ can remove the degeneracy, highlighting once more the simple nature of the problem. 
\begin{figure}
    \centering
    \includegraphics[width=0.47\textwidth]{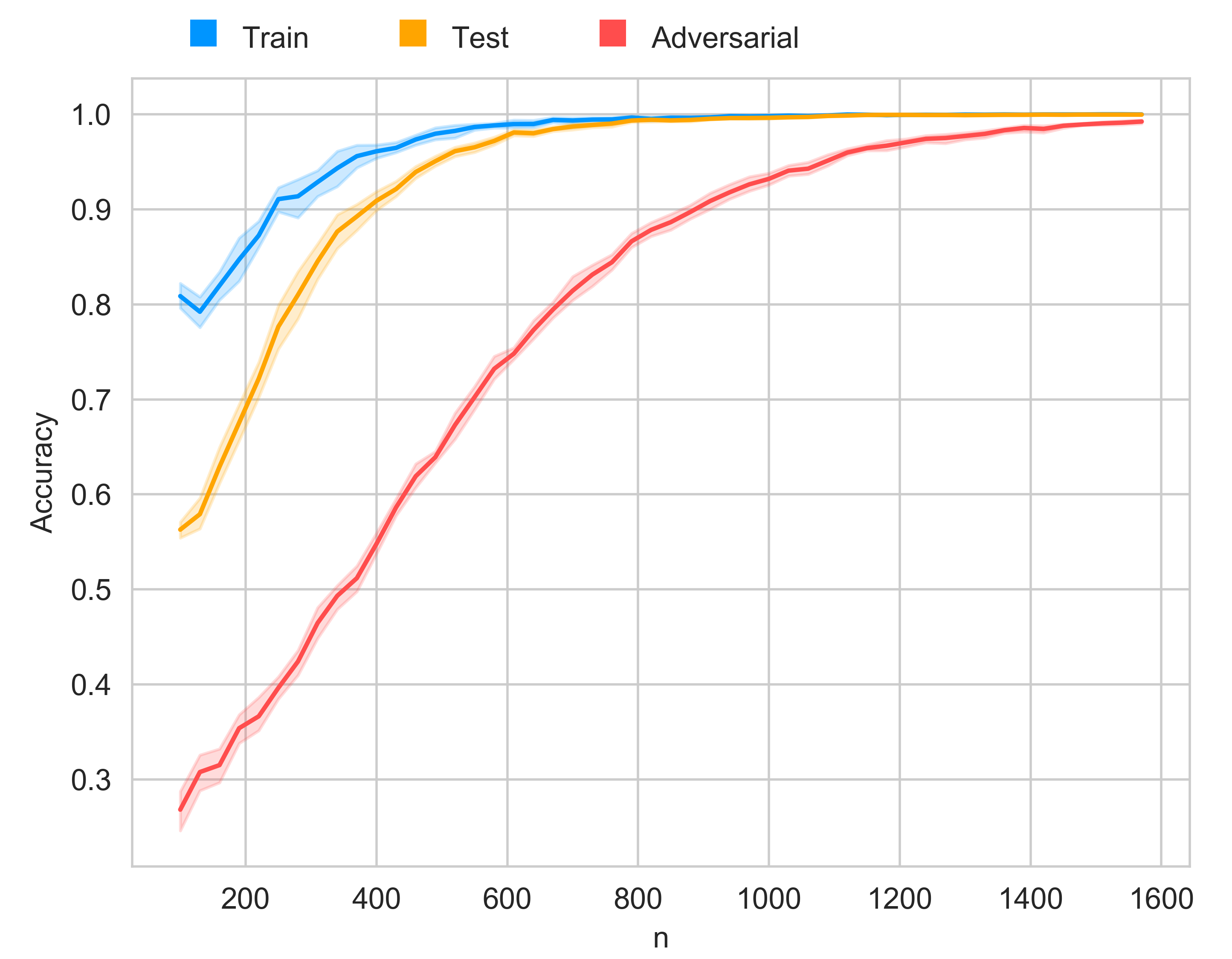}
    \vspace{-3mm}
    \caption{Train, test and adversarial accuracies plotted against sample size for the dominant eigenfunction of a 2-layer NNGP model. Results are averaged over $8$ runs.}
    \label{cutacc}
\end{figure}
\section{Discussion}
\label{discussion}
In this work, we provide a mathematical account of the adversarial phenomenon observed in \citet{nagarajan2019uniform}. We identified its origin, pin-pointing it to the output bias of the model which trades-off how much a network relies on radial information in the data. We studied the different phase transitions in the adversarial accuracy and linked them to a data-dependent quantity $\gamma_K(n)$ which we derived in closed-form for the expected kernel. Moreover, we studied how the adversarial effect behaves under eigendecompositions and showed numerically that even a restriction to the ideal eigenfunction does not alleviate the problem. The adversarial effect thus really is a consequence of the data distribution solely containing radial information, which in turn makes a neural network vulnerable if the output bias is not large enough. The problem observed in \citet{nagarajan2019uniform} does hence not point towards a deeper problem in the design of neural models or the optimizer and does not translate to other datasets directly.

\bibliographystyle{icml2021}
\bibliography{example_paper}
\onecolumn
\icmltitle{Appendix}




\vskip 0.3in



\begin{appendices}

\section{Omitted Proofs}
In this section we provide complete proofs of the results in the main text.
\subsection{Proof of Theorem 1}
\begin{theorem*}
Consider a fully-connected neural network with NTK parametrization as introduced in Section \ref{ntk}, equipped with a $1$-homogeneous activation function (such as ReLU). Set every bias to zero ($\beta_i = 0$) except for the output bias, $b^{(L)} \sim \mathcal{N}(0, \beta^2)$. Then it holds that both $\Theta^{(L)}$ and $\Sigma^{(L)}$ are semi-homogeneous kernels with $\zeta = \beta$.
\end{theorem*}
\begin{proof}
We will prove this statement via induction over the depth of the network. We will first show that a network without any bias is semi-homogeneous with parameter $\zeta = 0$. Fix any $\alpha \in \mathbb{R}_{+}$. Let us first consider the base case $l=1$. \\[3mm]
\textbf{Base case:} $\Sigma^{(1)}(\bm{x}, \bm{x}') = \Theta^{(1)}(\bm{x}, \bm{x}') = \frac{1}{d_0}\bm{x}^{T}\bm{x}'$.
We easily deduce that 
\begin{equation*}
    \begin{split}
        \Sigma^{(1)}(\alpha\bm{x}, \bm{x}') &= \frac{1}{d_0}(\alpha\bm{x})^{T}\bm{x}' + \beta^2  = \alpha \Sigma^{(1)}(\bm{x}, \bm{x}')
    \end{split}
\end{equation*}
The same thing holds for the NTK $\Theta^{(1)}$. The base case thus holds. \\[3mm]
\textbf{Induction step:} Assume that $\Sigma^{(l)}$ and $\Theta^{(l)}$ are semi-homogeneous with $\zeta = 0$. Let us first analyze the NNGP.
\begin{equation*}
    \begin{split}
        \Sigma^{(l+1)}(\alpha \bm{x}, \bm{x}') &= \mathbb{E}_{\bm{z} \sim \mathcal{N}(\bm{0},{\Sigma}^{(l)}{{|}}_{\alpha\bm{x}, \bm{x}'})}\big[\sigma(z_1)\sigma(z_2)\big] 
        = \frac{1}{2\pi\sqrt{\operatorname{det}({\Sigma}^{(l)}{{|}}_{\alpha\bm{x}, \bm{x}'})}}\int_{\mathbb{R}^2}\sigma(z_1)\sigma(z_2)e^{-\frac{1}{2}\bm{z}^{T}{\Sigma}^{(l)}{{|}}_{\alpha\bm{x}, \bm{x}'}^{-1}\bm{z}}d\bm{z}
    \end{split}
\end{equation*}
Now observe that 
\begin{equation*}
    \begin{split}
        \operatorname{det}\left(\bm{\Sigma}^{(l)}|_{\alpha\bm{x}, \bm{x}'}\right) &= {\Sigma}^{(l)}(\alpha\bm{x}, \alpha \bm{x}){\Sigma}^{(l)}(\bm{x}', \bm{x}')-{\Sigma}^{(l)}(\alpha \bm{x}, \bm{x}')^2 \\
        &= \alpha^2 \left({\Sigma}^{(l)}({\bm{x}}, {\bm{x}}){\Sigma}^{(l)}({\bm{x}'}, {\bm{x}'})-{\Sigma}^{(l)}({\bm{x}}, {\bm{x}'})^2\right) \\
        &= \alpha^2 \operatorname{det}\left(\bm{\Sigma}^{(l)}|_{{\bm{x}}, {\bm{x}'}}\right)
    \end{split}
\end{equation*}
On the other hand we have that 
\begin{equation*}
    \begin{split}
        \left(\bm{\Sigma}^{(l)}|_{\alpha\bm{x}, \bm{x}'}\right)^{-1} &= \frac{1}{\operatorname{det}\left(\bm{\Sigma}^{(l)}|_{\alpha\bm{x}, \bm{x}'}\right)} \begin{pmatrix} {\Sigma}^{(l)}({\bm{x}'}, {\bm{x}'}) & -{\Sigma}^{(l)}({\alpha \bm{x}}, {\bm{x}'}) \\
        -{\Sigma}^{(l)}({\alpha \bm{x}}, {\bm{x}'}) & {\Sigma}^{(l)}({\alpha \bm{x}}, {\alpha \bm{x}}) \end{pmatrix} 
        \\ &= \frac{1}{\alpha^2} \frac{1}{\operatorname{det}\left(\bm{\Sigma}^{(l)}|_{{\bm{x}}, {\bm{x}'}}\right)} \begin{pmatrix} {\Sigma}^{(l)}({\bm{x}'}, {\bm{x}'}) & -\alpha {\Sigma}^{(l)}({\bm{x}}, {\bm{x}'}) \\
        -\alpha {\Sigma}^{(l)}({\bm{x}}, {\bm{x}'}) & \alpha^2{\Sigma}^{(l)}({\bm{x}}, {\bm{x}}) \end{pmatrix}
    \end{split}
\end{equation*}
We can hence write that 
\begin{equation*}
    \begin{split}
        \bm{z}^{T}\left(\bm{\Sigma}^{(l)}|_{\alpha\bm{x}, \bm{x}'}\right)^{-1}\bm{z} &= \frac{1}{\operatorname{det}\left(\bm{\Sigma}^{(l)}|_{{\bm{x}}, {\bm{x}'}}\right)} \left(z_1^2\frac{1}{\alpha^2}{\Sigma}^{(l)}({\bm{x}'}, {\bm{x}'}) -2 z_1z_2\frac{1}{\alpha}{\Sigma}^{(l)}({\bm{x}}, {\bm{x}'}) + z_2^2{\Sigma}^{(l)}({\bm{x}}, {\bm{x}}) \right)
    \end{split}
\end{equation*}
Let us perform the substitution $u_1 = \frac{1}{\alpha} z_1$ and $u_2 = z_2$ with area element $d\bm{u} = \frac{1}{\alpha}d\bm{z} $. Then we can write 
\begin{equation*}
    \begin{split}
        \bm{z}^{T}\left(\bm{\Sigma}^{(l)}|_{\alpha\bm{x}, \bm{x}'}\right)^{-1}\bm{z} &= \frac{1}{\operatorname{det}\left(\bm{\Sigma}^{(l)}|_{{\bm{x}}, {\bm{x}'}}\right)} \left(u_1^2{\Sigma}^{(l)}({\bm{x}'}, {\bm{x}'}) -2u_1u_2 \Sigma^{(l)}({\bm{x}}, {\bm{x}'})  + u_2^2\Sigma^{(l)}({\bm{x}}, {\bm{x}})\right) \\
        &= \bm{u}^{T}\left(\bm{\Sigma}^{(l)}|_{\bm{x}, \bm{x}'}\right)^{-1}\bm{u}
    \end{split}
\end{equation*}
We can thus rewrite the integral as 
\begin{equation*}
    \begin{split}
        \Sigma^{(l+1)}(\alpha \bm{x}, \bm{x}') 
        &= \frac{1}{2\pi\sqrt{\operatorname{det}({\Sigma}^{(l)}{{|}}_{\alpha\bm{x}, \bm{x}'})}}\int_{\mathbb{R}^2}\sigma(z_1)\sigma(z_2)e^{-\frac{1}{2}\bm{z}^{T}{\Sigma}^{(l)}{{|}}_{\alpha\bm{x}, \bm{x}'}^{-1}\bm{z}}d\bm{z}\\ &=  \frac{1}{2\pi \alpha \sqrt{\operatorname{det}\left(\bm{\Sigma}^{(l)}|_{{\bm{x}}, {\bm{x}'}}\right)}}\int_{\mathbb{R}^2}\sigma(\alpha u_1)\sigma(u_2)e^{-\frac{1}{2}\bm{u}^{T}{\Sigma}^{(l)}{{|}}_{\alpha\bm{x}, \bm{x}'}^{-1}\bm{u}} \alpha d\bm{u} \\
        &= \alpha \Sigma^{(l+1)}(\bm{x}, \bm{x}')
    \end{split}
\end{equation*}
where we have used the $1$-homogenity of $\sigma$. Next we analyze the NTK $\Theta^{(l+1)}$. Here we have to control the additional term 
$$\dot{\Sigma}^{(l+1)}(\alpha\bm{x}, \bm{x}') = \mathbb{E}_{\bm{z} \sim \mathcal{N}\left(\bm{0}, \Sigma^{(l)}|_{\alpha\bm{x},\bm{x}'}\right)}\left[\dot{\sigma}(z_1)\dot{\sigma}(z_2)\right]$$
Since $\sigma$ is $1$-homogeneous, we know that its derivative is $0$-homogeneous. We can thus apply the exact same computation as for $\Sigma^{(l+1)}$ to arrive at
$$\dot{\Sigma}^{(l+1)}(\alpha\bm{x}, \bm{x}') =\dot{\Sigma}^{(l+1)}(\bm{x}, \bm{x}')$$
Using the previous result and the induction hypothesis, we obtain
\begin{equation*}
    \begin{split}
        \Theta^{(l+1)}(\alpha \bm{x}, \bm{x}') &= \Theta^{(l)}(\alpha \bm{x}, \bm{x}') \dot{\Sigma}^{(l+1)}(\alpha \bm{x}, \bm{x}') + \Sigma^{(l+1)}(\alpha \bm{x}, \bm{x}')=\alpha\Theta^{(l)}( \bm{x}, \bm{x}') \dot{\Sigma}^{(l+1)}(\bm{x}, \bm{x}') + \alpha\Sigma^{(l+1)}( \bm{x}, \bm{x}') \\ &= \alpha \Theta^{(l+1)}(\bm{x}, \bm{x}')
    \end{split}
\end{equation*}
Given this result, we can now consider the kernel with an output bias $\beta$ added. Let $K$ denote either the NTK or NNGP kernel with an output bias and $C$ the corresponding kernel without output bias. Then we obtain
\begin{equation*}
    \begin{split}
        K(\alpha \bm{x}, \bm{x}') &= C(\alpha \bm{x}, \bm{x}') + \beta^2 = \alpha C(\bm{x}, \bm{x}') + \beta^2 = \alpha \left(C(\bm{x}, \bm{x}') + \beta^2 - \beta^2\right) + \beta^2 \\
        &= \alpha K(\bm{x}, \bm{x}') + \beta^2(1-\alpha)
    \end{split}
\end{equation*}
This concludes the proof.
\end{proof}
\subsection{Proof of Lemma 2}
\begin{lemma*}
Fix a semi-homogeneous kernel $K$ and two data points sampled according to the adversarial spheres measure, $\bm{x}, \bm{z} \sim p$. Consider the projection $\mathcal{P}(\bm{x})$. Denote $r=||\bm{x}||_2$ and $\tilde{r} = ||\mathcal{P}(\bm{x})||_2$. Then it holds that:
 $$K(\mathcal{P}(\bm{x}), \bm{z}) = \frac{\tilde{r}}{r}K(\bm{x}, \bm{z}) + \zeta^2(1-\frac{\tilde{r}}{r})$$
\end{lemma*}
\begin{proof}
Realize that we can write the projection as 
$$\mathcal{P}(\bm{x}) = \frac{\tilde{r}}{r}\bm{x} = \alpha \bm{x}$$
Obviously, $\alpha > 0$, thus we can apply the defining property of semi-homogeneous kernels to conclude
\begin{equation*}
    \begin{split}
    K(\mathcal{P}(\bm{x}), \bm{z}) = K(\alpha \bm{x}, \bm{z}) = \frac{\tilde{r}}{r}K(\bm{x}, \bm{z}) + \zeta^2(1-\frac{\tilde{r}}{r})
    \end{split}
\end{equation*}
\end{proof}
\subsection{Proof of Corollary 2.1}
\begin{corollary*}
Fix a semi-homogeneous kernel $K$ and a data point sampled according to the adversarial spheres measure, $\bm{x} \sim p$. Consider the projection $\mathcal{P}(\bm{x})$. Denote $r=||\bm{x}||_2$ and $\tilde{r} = ||\mathcal{P}(\bm{x})||_2$. Then it holds that 
$$f_K\left(\mathcal{P}(\bm{x})\right) = \frac{\tilde{r}}{r}f_K(\bm{x}) + \zeta^{2}\left(1- \frac{\tilde{r}}{r}\right)\gamma_K(n)$$
where we define $\gamma_K(n) = \bm{1}_n^{T}K(\bm{X}, \bm{X})^{-1}\bm{y}$ and $\bm{1}_n = \left(1, \dots, 1\right)^{T} \in \mathbb{R}^{n}$.
\end{corollary*}
\begin{proof}
We just need to apply the Lemma \ref{semikernel}:
\begin{equation*}
    \begin{split}
        f_K\left(\mathcal{P}(\bm{x})\right) &= K(\mathcal{P}(\bm{x}), \bm{X})\left(K(\bm{X}, \bm{X})\right)^{-1}\bm{y} \\
        &\stackrel{}{=} \left(\frac{\tilde{r}}{r} K(\bm{x}, \bm{X}) + \zeta^2\left(1-\frac{\tilde{r}}{r}\right)\bm{1}_n\right)\left(K(\bm{X}, \bm{X})\right)^{-1}\bm{y} \\
        &= \frac{\tilde{r}}{r}f_K(\bm{x}) + \zeta^2\left(1-\frac{\tilde{r}}{r}\right)\bm{1}_n\left(K(\bm{X}, \bm{X})\right)^{-1}\bm{y} \\
        &=\frac{\tilde{r}}{r}f_K(\bm{x}) + \zeta^2\left(1-\frac{\tilde{r}}{r}\right) \gamma_K(n)
    \end{split}
\end{equation*}
\end{proof}
\subsection{Proof of Theorem 3}
\begin{theorem*}
Take a semi-homogeneous kernel K and consider a training set $\mathcal{S}_{\text{train}} \stackrel{\text{i.i.d.}}{\sim} \mathcal{D}^n$ along with the corresponding adversarial set $\mathcal{S}_{\text{adv}}$. Then it holds that $a_{\text{adv}}$ is quantized to only three values:
$$a_{\text{adv}} \in \Big{\{}0, 1-q, 1\Big{\}}$$
Moreover, we can characterize the phase transitions in sample size $n$ as 
$$a_{\text{adv}} = \begin{cases} 0 \hspace{8mm}\text{ if } \hspace{3mm} \gamma_K(n) \leq \frac{r_1}{\zeta^2(r_2-r_1)} \\[2mm]
1- q \hspace{2mm}\text{ if } \hspace{2mm} \frac{r_1}{\zeta^2(r_2-r_1)} \leq \gamma_K(n) \leq \frac{r_2}{\zeta^2(r_2-r_1)} \\[2mm]
1 \hspace{8mm}\text{ if }\hspace{3mm} \gamma_K(n) \geq \frac{r_2}{\zeta^2(r_2-r_1)} \end{cases}$$
\end{theorem*}
\begin{proof}
By an extension of Proposition 2 in \textcolor{blue}{Jacot et al. (2018)}, we know that the inverse of kernel matrix $K(\bm{X}, \bm{X})$ is well-defined, implying that we have perfect training accuracy:
$$f_K(\bm{X}) = \bm{y}$$
\begin{equation*}
    \begin{split}
        f_K\left(\mathcal{P}(\bm{x})\right) &= \frac{\tilde{r}}{r}f_K(\bm{x}) + \zeta^{2}\left(1- \frac{\tilde{r}}{r}\right)\gamma_K(n) \\
        &= \frac{\tilde{r}}{r}y + \zeta^{2}\left(1- \frac{\tilde{r}}{r}\right)\gamma_K(n)
    \end{split}
\end{equation*}
Assume first that $y=1$, implying $||\bm{x}||_2 = r_1$. Thus we need that 
\begin{equation*}
    \begin{split}
        &  \frac{{r_2}}{r_1} + \zeta^{2}\left(1- \frac{{r_2}}{r_1}\right)\gamma_K(n) < 0 \\
        & \hspace{-10mm}\iff \gamma_K(n) > \frac{r_2}{\zeta^2(r_2-r_1)}
    \end{split}
\end{equation*}
The case $y=-1$ is similarly obtained. Notice that the inequality is entirely independent of the specific form of $\bm{x}$. Thus this inequality will hold for all $\bm{x}$ with label $y=1$ simultaneously. Since $r_1 < r_2$, the part of the adversarial data with label $y_{\text{adv}} = 1$ (or clean label $y=-1$) will be learnt first. This corresponds to a fraction of $1-q$ of the entire training set, leading to $1-q$ correctly classified adversarial examples.
\end{proof}
\subsection{Proof of Lemma 4}
\begin{lemma*}
Assume that $K$ is of the above form and denote by $C$ the corresponding homogeneous kernel. Then it holds that 
$$\gamma_{K}(n) =  \frac{1}{1 + \beta^2{s}\left({C}(\bm{X},\bm{X})^{-1}\right)}\gamma_{{C}}(n)$$
where we define ${s}(\bm{A}) = \sum_{i,j}A_{ij}$.
\end{lemma*}
\begin{proof}
Using the Sherman–Morrison formula, we expand the inverse as follows:
\begin{equation*}
    \begin{split}
        \gamma_K(n) &= \bm{1}_n^{T}K(\bm{X}, \bm{X})^{-1}\bm{y} =   \bm{1}_n^{T}\left(C(\bm{X}, \bm{X}) + \beta^2 \bm{1}_n \bm{1}_n^{T}\right)^{-1}\bm{y} = \bm{1}_n^{T}\left(C(\bm{X}, \bm{X})^{-1} + \right)\\ &= \bm{1}_n^{T}C(\bm{X}, \bm{X})^{-1}\bm{y} -\bm{1}_n\left( \frac{\beta^2C(\bm{X}, \bm{X})^{-1}\bm{1}_n\bm{1}_n^{T}C(\bm{X}, \bm{X})^{-1}}{1 + \beta^2 \bm{1}_n^{T}C(\bm{X}, \bm{X})^{-1}\bm{1}_n}\right)\bm{y} \\
        &= \gamma_C(n) - \frac{\beta^2 s\left(C(\bm{X}, \bm{X})^{-1}\right)\gamma_C(n)}{1+\beta^2s\left(C(\bm{X}, \bm{X})^{-1}\right)} \\
        &= \frac{\gamma_C(n)}{1+\beta^2s\left(C(\bm{X}, \bm{X})^{-1}\right)}
    \end{split}
\end{equation*}
\end{proof}
\subsection{Proof of Theorem 5}
\begin{theorem*}
Consider the expected kernel $\tilde{K} = \mathbb{E}_{\bm{X} \sim p^{n}}\left[K(\bm{X}, \bm{X})\right]$. We have that $\gamma_{\tilde{K}}$ is asymptotically given by  
$$\gamma_{\tilde{K}}(n) \propto  \frac{C_1 + \eta n}{C_2 - \beta^2C_3 n}$$
for constants $C_1, C_2, C_3, \eta \in \mathbb{R}$ and the limit is given by
$$\gamma_{\tilde{K}}(n) \xrightarrow[]{n \xrightarrow[]{}\infty}\frac{r_1+r_2}{\beta^2(r_2-r_1)}$$
\end{theorem*}
\begin{proof}
Define the quantities $\alpha = K(\bm{e}_1, \bm{e}_1)$ where $\bm{e}_1$ denotes the first unit vector. Notice that since $K$ is a dot-product kernel, it holds for any $\bm{x} \in \bm{S}_1^{d-1}$ that $K(\bm{x}, \bm{x}) = \alpha$. Moreover, define $\rho = \mathbb{E}_{\bm{x}, \bm{x}' \sim p_1}\left[K(\bm{x}, \bm{x}')\right] $ and consider the expected kernel $\tilde{\bm{K}} = \mathbb{E}_{\bm{X} \sim p^n}[K(\bm{X}, \bm{X})]$. Recall that we assume a balanced dataset, where the upperhalf of $\bm{X}$ is sampled according to $p_{r_1}$ and the second half according to $p_{r_2}$. Denote the respective samples by $\bm{X}_{+}$ and $\bm{X}_{-}$ Let us first focus on the bias-free part $\gamma_C(n)$. The bias-free kernel is given by the following block structure:
\begin{equation*}
    \begin{split}
        C(\bm{X}, \bm{X}) &= \begin{pmatrix} C(\bm{X}_{+},\bm{X}_{+}) &  C(\bm{X}_{+},\bm{X}_{-}) \\[2mm]
 C(\bm{X}_{-},\bm{X}_{+}) &  C(\bm{X}_{-},\bm{X}_{-})
\end{pmatrix} = \begin{pmatrix} r_1^2C(\tilde{\bm{X}}_{+},\tilde{\bm{X}}_{+}) &  r_1r_2C(\tilde{\bm{X}}_{+},\tilde{\bm{X}}_{-}) \\[2mm]
 r_1r_2C(\tilde{\bm{X}}_{-},\tilde{\bm{X}}_{+}) &  r_2^2C(\tilde{\bm{X}}_{-},\tilde{\bm{X}}_{-})  
\end{pmatrix} \\[2mm] &:= \begin{pmatrix} r_1^2\bm{C}_{++} &  r_1r_2\bm{C}_{+-} \\[2mm]
 r_1r_2\bm{C}_{-+} &  r_2^2\bm{C}_{--}
\end{pmatrix}
    \end{split}
\end{equation*}

where $\tilde{\bm{X}}$ denotes the projected data to the unit sphere $\bm{S}^{d-1}_1$. We will be interested in the blocks of the inverse $C(\bm{X}, \bm{X})^{-1}$:
$$C(\bm{X}, \bm{X})^{-1} = \begin{pmatrix} \bm{A} &  \bm{B} \\[2mm]
 \bm{B} &  \bm{D}
\end{pmatrix}$$
Due to the symmetry, we observe that 
$$\bm{1}_n^{T}C(\bm{X}, \bm{X})^{-1}\bm{y} = s(\bm{A})-s(\bm{D})$$
By the inverse formula for block matrices, we can analyse the first inverse block to obtain that 
\begin{equation*}
    \begin{split}
        \bm{A} &= \frac{1}{r_1^2}\bm{C}_{++}^{-1} + \frac{1}{r_1^2}\bm{C}_{++}^{-1}{r_1r_2} \bm{C}_{+-}\left(r_2^2\bm{C}_{--}-r_1^2r_2^2\bm{C}_{-+}\frac{1}{r_1^2}\bm{C}_{++}^{-1}\bm{C}_{+-}\right)^{-1}{r_1r_2} \bm{C}_{-+}\frac{1}{r_1^2}\bm{C}_{++}^{-1} \\
        &=\frac{1}{r_1^2}\left(\bm{C}_{++}^{-1}+\bm{C}_{++}^{-1} \bm{C}_{+-}\left(\bm{C}_{--}-\bm{C}_{-+}\bm{C}_{++}^{-1}\bm{C}_{+-}\right)^{-1} \bm{C}_{-+}\bm{C}_{++}^{-1}\right)
    \end{split}
\end{equation*}
Next we analyze the right bottom block of the inverse:
\begin{equation*}
    \begin{split}
        \bm{D} &=  \left(r_2^2\bm{C}_{--}-r_1^2r_2^2\bm{C}_{-+}\frac{1}{r_1^2}\bm{C}_{++}^{-1}\bm{C}_{+-}\right)^{-1} = \frac{1}{r_2^2}\left(\bm{C}_{--}-\bm{C}_{-+}\bm{C}_{++}^{-1}\bm{C}_{+-}\right)^{-1}
    \end{split}
\end{equation*}
Similarly, we simplify the off-diagonal blocks to
\begin{equation*}
    \begin{split}
        \bm{B} &= -\frac{1}{r_2^2}\left(\bm{C}_{--}-\bm{C}_{-+}\bm{C}_{++}^{-1}\bm{C}_{+-}\right)^{-1}r_1r_2 \bm{C}_{-+}\frac{1}{r_1^2}\bm{C}_{++}^{-1} \\&= \frac{1}{r_1r_2}\left(\bm{C}_{--}-\bm{C}_{-+}\bm{C}_{++}^{-1}\bm{C}_{+-}\right)^{-1}\bm{C}_{-+}\bm{C}_{++}^{-1}
    \end{split}
\end{equation*}
By introducing the inverse of the projected kernel matrix
$$C(\bar{\bm{X}}, \bar{\bm{X}})^{-1} = \begin{pmatrix} \bar{\bm{A}} &  \bar{\bm{B}} \\[2mm]
 \bar{\bm{B}} &  \bar{\bm{D}}
\end{pmatrix}$$
we quickly realize again through the inverse block matrix formula, that 
$$C(\bm{X}, \bm{X})^{-1} =  \begin{pmatrix} \frac{1}{r_1^2}\bar{\bm{A}} &  \frac{1}{r_1r_2}\bar{\bm{B}} \\[2mm]
 \frac{1}{r_1r_2}\bar{\bm{B}} &  \frac{1}{r_2^2}\bar{\bm{D}}
\end{pmatrix}$$
Thus we can see that 
$$\bm{1}_n^{T}C(\bm{X}, \bm{X})^{-1}\bm{y} = \frac{1}{r_1^2}s(\bar{\bm{A}})-\frac{1}{r_2^2}s(\bar{\bm{D}})$$
Let us now consider the expected kernel. Since both $\tilde{\bm{A}}$ and $\tilde{\bm{D}}$ are the Gram matrix of the same kernel on the unit sphere, their expected kernel agrees and thus we only need to calculate $s(\tilde{\bm{D}})$.
We define
$$\bm{G} = (\alpha - \rho)\mathds{1}_{m \times m} + \rho \bm{1}_m \bm{1}_m^{T}$$
as well as the expected off-diagonal part
$$\bm{H} = \rho \bm{1}_m \bm{1}_m^{T}$$
Consider the matrix 
$$\tilde{\bm{D}}^{-1} =  \bm{C}_{--}-\bm{C}_{-+}\bm{C}_{++}^{-1}\bm{C}_{+-}$$
In expectation, this reduces to 
$$\bm{S} := \mathbb{E}\left[\tilde{\bm{D}}^{-1}\right] =  \bm{G} - \bm{H}\bm{G}^{-1}\bm{H}$$
Again, making use of Sherman-Morrison, we can find a closed form expression for $\bm{G}^{-1}$:
\begin{equation*}
    \begin{split}
        \bm{G}^{-1} &= \left((\alpha - \rho)\mathds{1}_{m \times m} + \rho \bm{1}_m \bm{1}_m^{T}\right)^{-1} = \frac{1}{\alpha-\rho}\mathds{1}_{m \times m} - \frac{\rho}{(\alpha-\rho)\left(\alpha+\rho(m-1)\right)}\bm{1}_m\bm{1}_m^{T}
    \end{split}
\end{equation*}
and thus we can simplify 
\begin{equation*}
    \begin{split}
        \bm{H}\bm{G}^{-1}\bm{H} &=  \frac{\rho^2m}{\alpha-\rho}\bm{1}_m\bm{1}_m^{T} - \frac{m^2\rho^{3}}{(\alpha-\rho)\left(\alpha+\rho(m-1)\right)}\bm{1}_m\bm{1}_m^{T} =\frac{(\alpha + \rho(m-1))\rho^2m - m^2\rho^3}{(\alpha-\rho)\left(\alpha+\rho(m-1)\right)} \bm{1}_m\bm{1}_m^{T}\\
        &= \frac{m\rho^2}{\alpha+\rho(m-1)} \bm{1}_m\bm{1}_m^{T}
    \end{split}
\end{equation*}
We can now show that $\bm{1}_m$ is an eigenvector of $\bm{S}$:
\begin{equation*}
\begin{split}
    \mathbb{E}\left[\tilde{\bm{D}}^{-1}\right]\bm{1}_m &= (\alpha - \rho)\bm{1}_m + m\rho \bm{1}_m  - \frac{m^2\rho^2}{\alpha+\rho(m-1)} \bm{1}_m \\ 
    &= \frac{\left(\alpha+\rho(m-1)\right)^2 -m^2\rho^2}{\left(\alpha+\rho(m-1)\right)}\bm{1}_m  \\
    &= \frac{(\alpha-\rho)^2+2m\rho(\alpha-\rho)}{\left(\alpha+\rho(m-1)\right)}\bm{1}_m 
\end{split}
\end{equation*}
Now we know that $\bm{1}_m$ is also an eigenvector of $\bm{S}^{-1}$ with inverse eigenvalue and thus
$$s\left(\bm{S}^{-1}\right) = \bm{1}_m^{T}\bm{S}^{-1}\bm{1}_m = \frac{\left(\alpha+\rho(m-1)\right)}{(\alpha-\rho)^2+2m\rho(\alpha-\rho)}\bm{1}_m^{T}\bm{1}_m = \frac{\rho m^2 + m(\alpha- \rho)}{(\alpha-\rho)^2+2m\rho(\alpha-\rho)}$$
Using the symmetry, we thus proved that 
$$\gamma_{\tilde{\bm{C}}}(m) = \frac{r_2^2-r_1^2}{r_1^2r_2^2}\frac{\rho m^2 + m(\alpha- \rho)}{(\alpha-\rho)^2+2m\rho(\alpha-\rho)} \propto \frac{r_2^2-r_1^2}{r_1^2r_2^2} \frac{\rho m^2}{(\alpha-\rho)^2+2m\rho(\alpha-\rho)}$$
To finish the proof, we need to finally calculate $s\left(\tilde{\bm{C}}^{-1}\right)$. Luckily, since we already calculated the sum of the block diagonal, we only need to find the sum of the off-diagonal blocks, in expectation:
$$\bm{W} = \bm{S}^{-1}\bm{H}{\bm{G}}^{-1}$$
Again we find the that $\bm{1}_m$ is an eigenvector:
\begin{equation*}
    \begin{split}
        \bm{W}\bm{1}_m &= -\bm{S}^{-1}\bm{H}\frac{1}{\alpha + \rho (m-1)}\bm{1}_m = -\bm{S}^{-1}\frac{m\rho}{\alpha + \rho (m-1)}\bm{1}_m = -\frac{m\rho}{\alpha + \rho (m-1)}\frac{\left(\alpha+\rho(m-1)\right)}{(\alpha-\rho)^2+2m\rho(\alpha-\rho)} \bm{1}_m \\
        &= -\frac{m\rho}{(\alpha-\rho)^2+2m\rho(\alpha-\rho)}\bm{1}_m
    \end{split}
\end{equation*}
where we used that $\bm{1}_m$ is also an eigenvector of $\bm{G}^{-1}$, as seen above. Thus the sum of the off-diagonal term is 
$$s(\bm{W}) = -\frac{m^2\rho}{(\alpha-\rho)^2+2m\rho(\alpha-\rho)}$$
Thus we can finally obtain that the sum of the inverse expected kernel $\tilde{\bm{C}}^{-1} = C(\bm{X}, \bm{X})^{-1}$ is given by
\begin{equation*}
    \begin{split}
        s\left(\tilde{\bm{C}}^{-1}\right) &= \left(\frac{1}{r_1} + \frac{1}{r_2}\right)s(\tilde{\bm{D}})+\frac{2}{r_1r_2}s(\bm{W})=\frac{1}{(\alpha-\rho)^2+2m\rho(\alpha-\rho)}\left(m^2\rho \frac{(r_1-r_2)^2}{r_1^2r_2^r} + \frac{r_1^2+r_2^2}{r_1^2r_2^2}m(\alpha-\rho)\right)  \\
        & \propto \frac{m^2\rho \frac{(r_1-r_2)^2}{r_1^2r_2^r}}{(\alpha-\rho)^2+2m\rho(\alpha-\rho)}
    \end{split}
\end{equation*}
Now, we can put all the pieces together to obtain
\begin{equation*}
    \begin{split}
        \gamma_{\tilde{\bm{K}}}(m) &= \frac{\gamma_{\tilde{\bm{C}}}(m)}{1+\beta^2s\left(\tilde{\bm{C}}^{-1}\right)} =  \frac{\frac{r_2^2-r_1^2}{r_1^2r_2^2}\left(\rho m^2 + m (\alpha-\rho)\right)}{(\alpha-\rho)^2 + 2m\rho(\alpha-\rho) + \beta^2m^2\rho\frac{(r_1-r_2)^2}{r_1^2r_2^r} + \beta^2\frac{r_1^2+r_2^2}{r_1^2r_2^2}m(\alpha-\rho)} \\
        & \propto \frac{m \rho \frac{r_2^2-r_1^2}{r_1^2r_2^2} + (\alpha-\rho)\frac{r_2^2-r_1^2}{r_1^2r_2^2}}{m\beta^2\rho\frac{(r_1-r_2)^2}{r_1^2r_2^r} + \beta^2\frac{r_1^2+r_2^2}{r_1^2r_2^2}(\alpha-\rho)}\\
        &= \frac{C_1 + \eta m}{C_2 - \beta^2C_3 m}
    \end{split}
\end{equation*}
Moreover, we can easily derive the limit
$$\gamma_{\tilde{\bm{K}}}(m) \xrightarrow[]{m \xrightarrow[]{}\infty}\frac{}{}\frac{r_1+r_2}{\beta^2(r_2-r_1)}$$
\end{proof}
\newpage
\section{Additional Numerical Experiments}
\subsection{Adversarial Accuracy}
Here we present more empirical evidence backing the theoretical findings in Theorem 3 for more kernels and architectures.
\begin{figure}[h]
    \centering
    \includegraphics[width=0.3\textwidth]{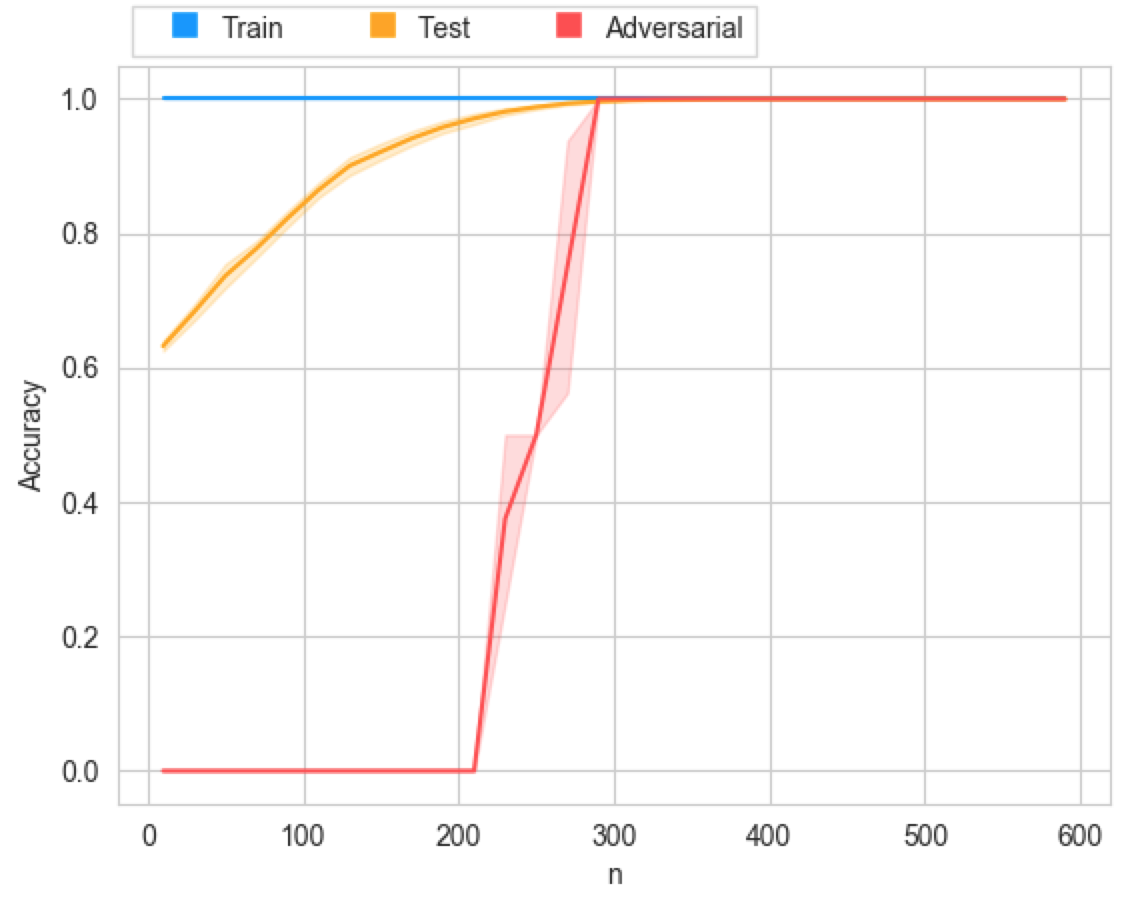}%
    \includegraphics[width=0.3\textwidth]{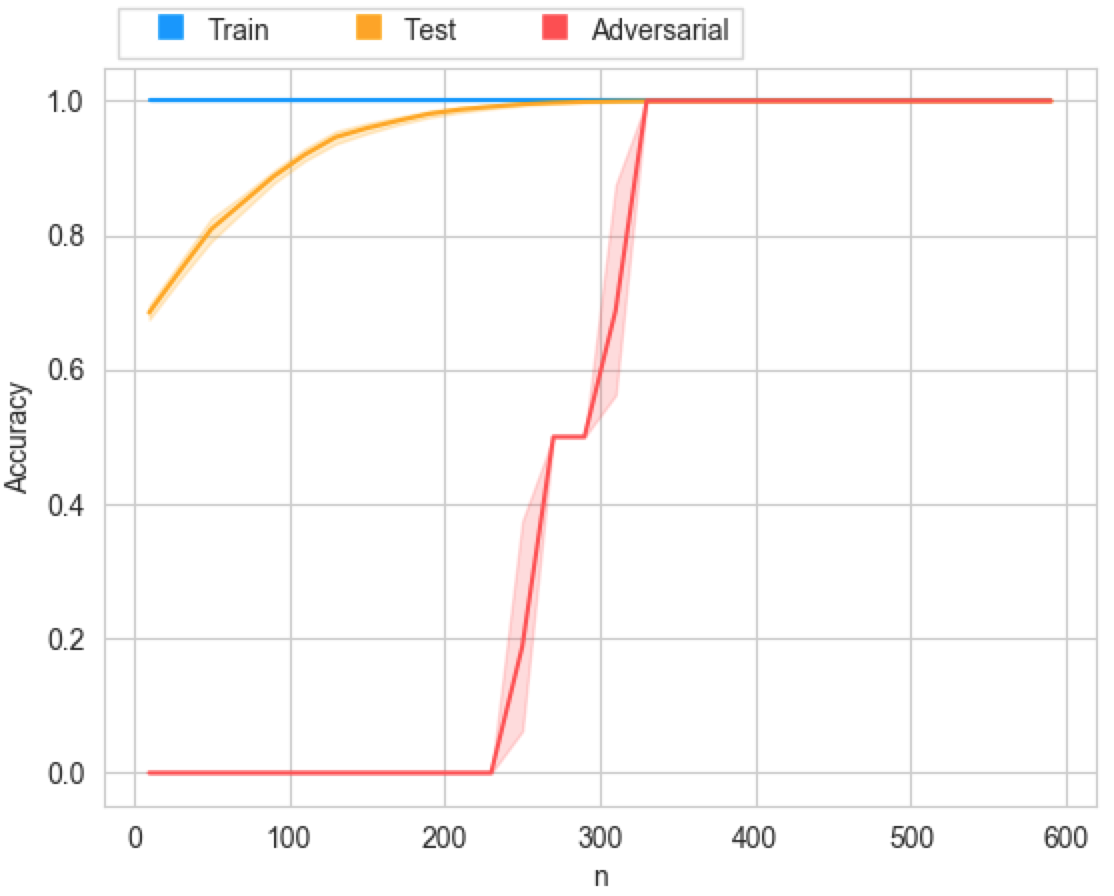}%
    \includegraphics[width=0.3\textwidth]{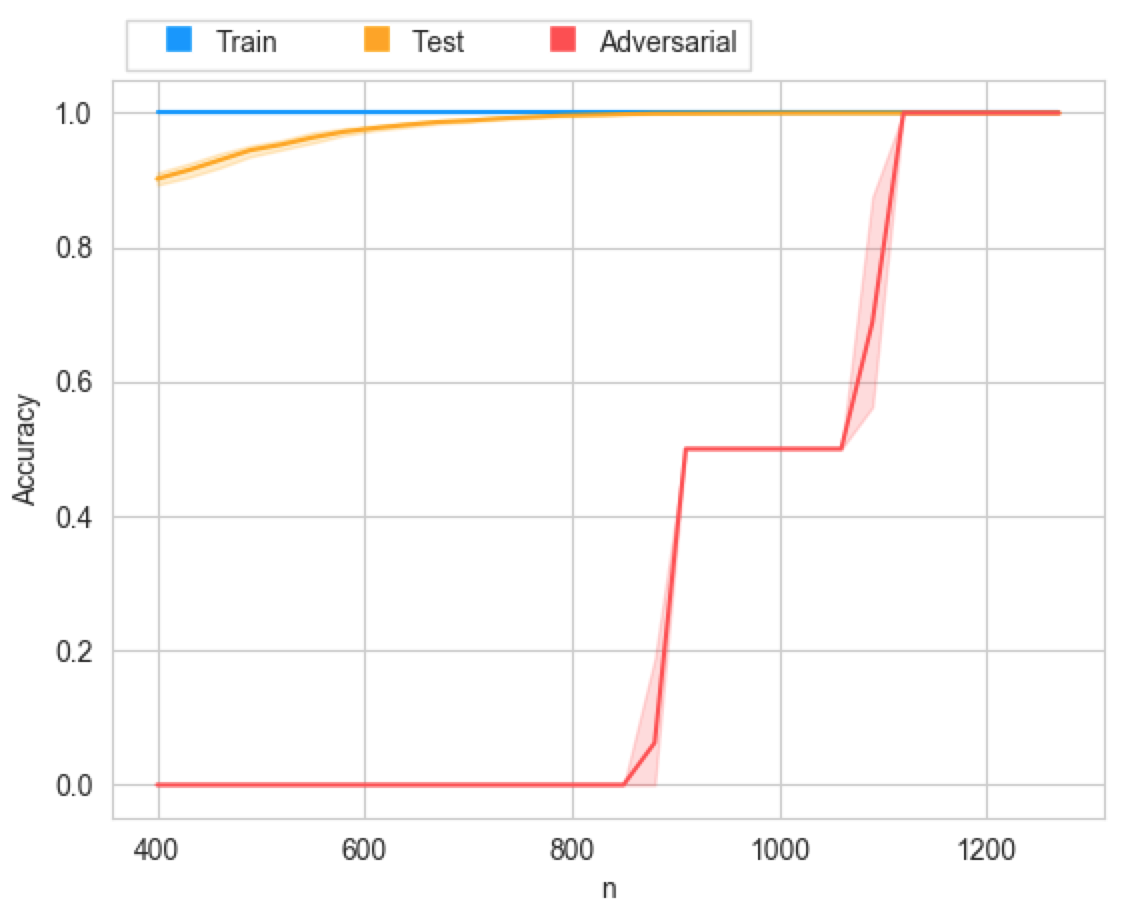}
    \caption{Train, test and adversarial accuracy for 5 layer NNGP (left), 9 layer NNGP (middle) and 5 layer NTK (right), plotted against sample size.} 
    \label{advaccapp}
\end{figure}
As demonstrated in \ref{advaccapp}, the phase transitions in the  adversarial accuracy hold across different underlying architectures. We observe that the NTK in general is suffering more from the adversarial effect, compared to the NNGP. This is also visible in Figure \ref{kernelgamma} where we see that $\gamma_K(n)$ grows more slowly for the NTK, compared to the NNGP, making it hence also more slowly approach the phase transitions. 
\subsection{Behaviour of $\gamma_K$}
\label{gamma_behaviour}
We study the behavour of $\gamma_K$ and the corresponding expected version $\gamma_{\tilde{\bm{K}}}$ for different architectures in Figure \ref{kernelgamma}
\begin{figure}[h]
    \centering
    \includegraphics[width=0.3\textwidth]{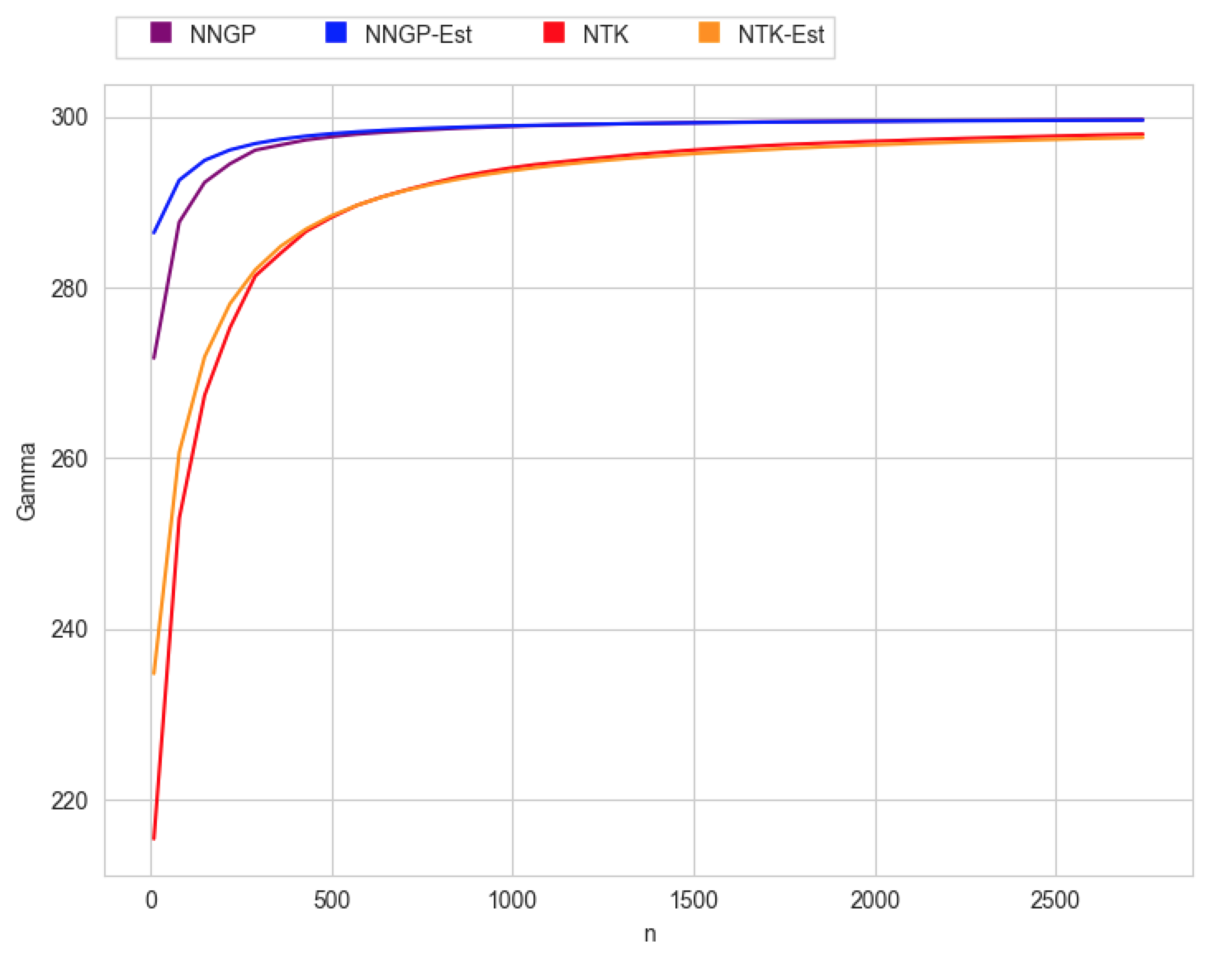}%
    \includegraphics[width=0.3\textwidth]{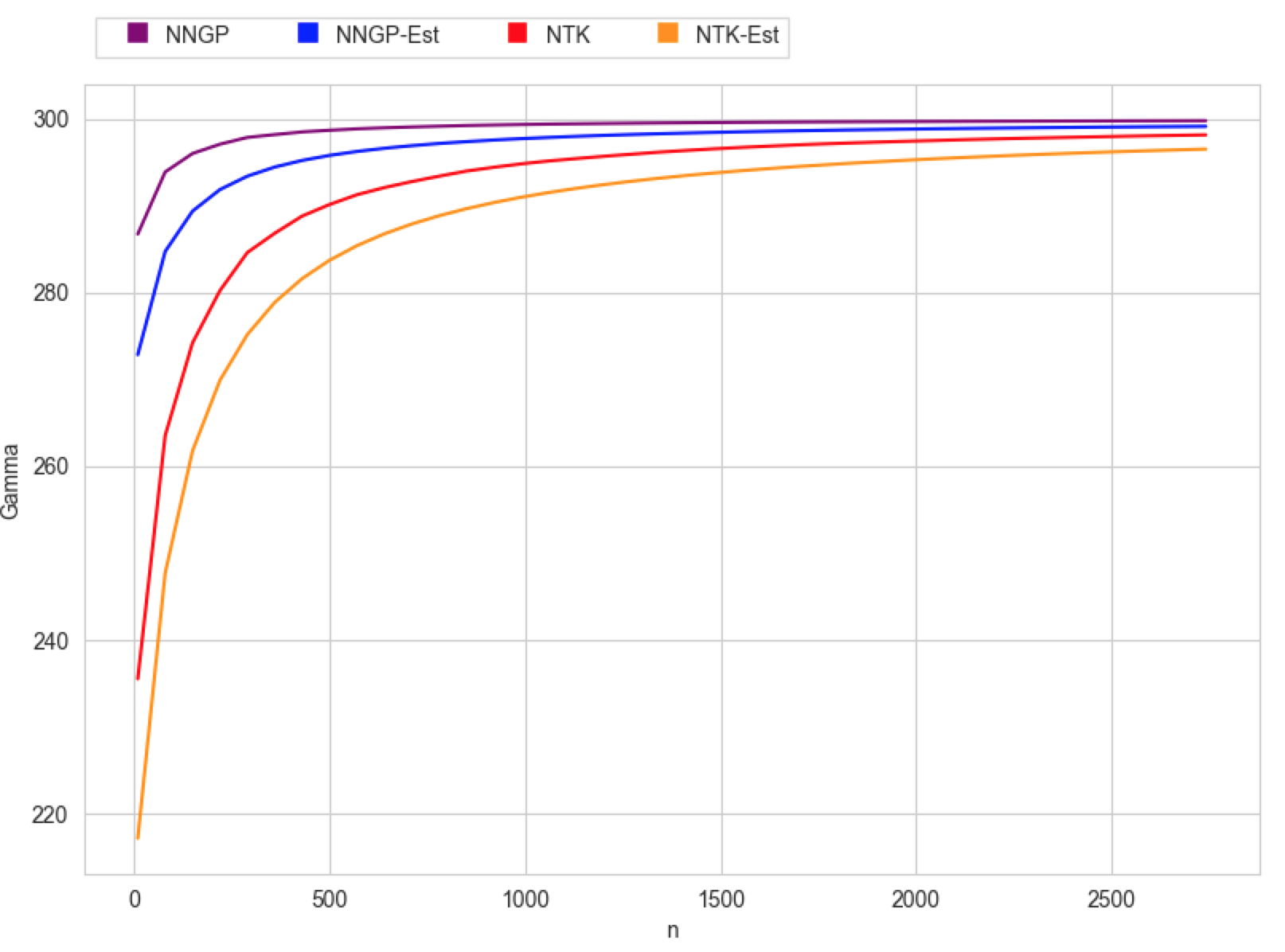}%
    \includegraphics[width=0.3\textwidth]{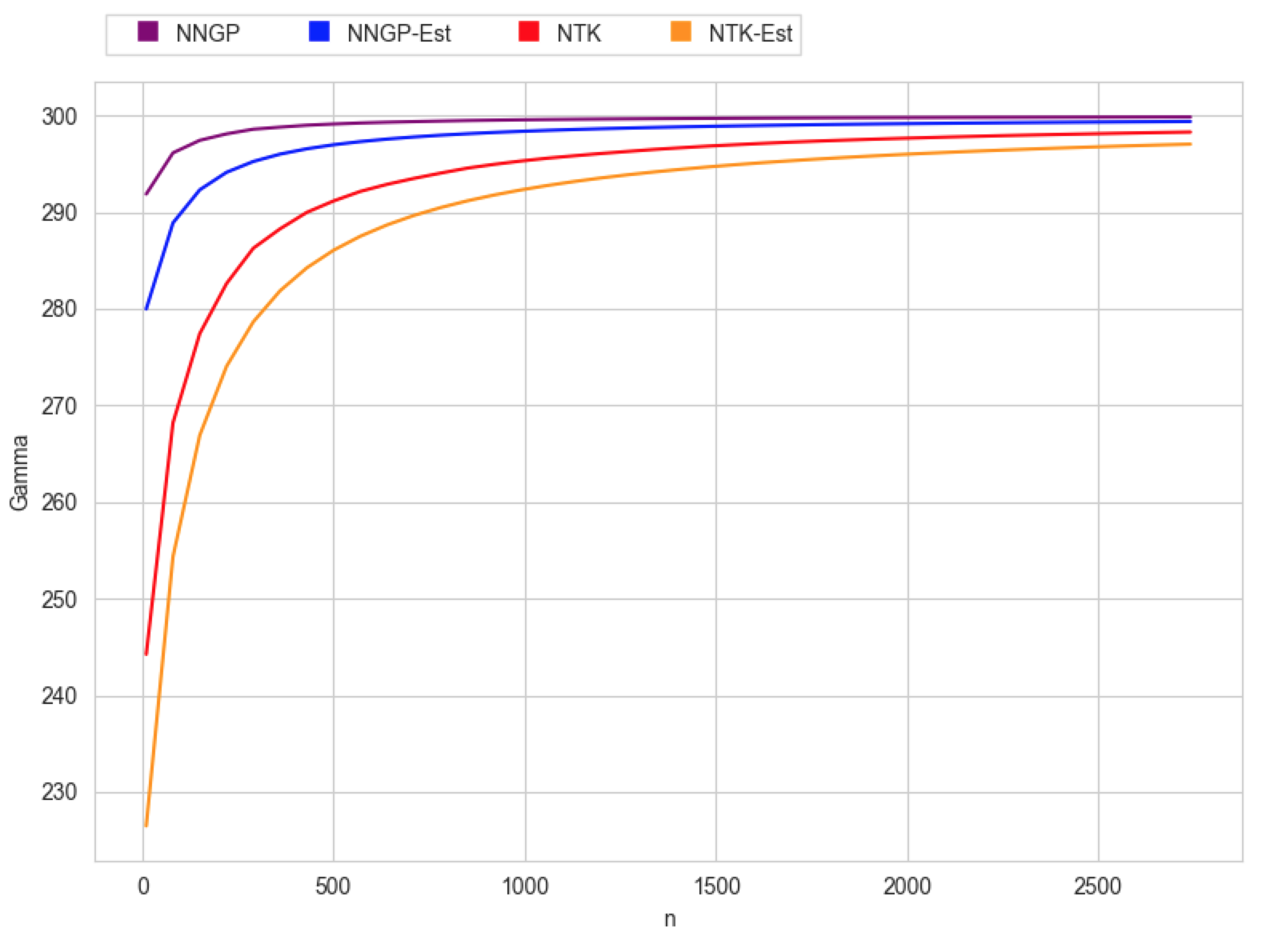}
    \caption{We plot $\gamma_K$ and $\gamma_{\tilde{\bm{K}}}$ for  a 3 layer NTK and NNGP (left), a 5 layer NTK and NNGP (middle) and a 7 layer NTK and NNGP (right)}
    \label{kernelgamma}
\end{figure}
We see that the expected kernel induces a very good approximation $\gamma_{\tilde{\bm{K}}}$, especially for large sample sizes $n$. Moreover, the qualitative behaviour is also very well captured for smaller sample sizes. Transforming insights from $\gamma_{\tilde{\bm{K}}}$ to $\gamma_{{K}}$ is thus sensible, especially for large sample sizes. Moreover, as anticipated in Theorem 5, all the kernels are converging to the same maximal capacity
$$\gamma_{\tilde{\bm{K}}}(m) \xrightarrow[]{m \xrightarrow[]{}\infty}\frac{}{}\frac{r_1+r_2}{\beta^2(r_2-r_1)}$$
\subsection{Eigendecompositions}
\label{more_eigen}
Here we study different decompositions, not just consisting of the dominant eigenfunction. In Figure \ref{decomp} we verify that the dominant eigenfunction indeed captures all the signal in the data, leaving the ensemble of eigenfunctions consisting of all but the dominant one with no predictive power at all in terms of any accuracy. We then proceed to study if using the top $10$ dominant eigenfunction brings any improvement in terms of the adversarial accuracy. Again this is not the case as visible in Figure \ref{decomp}. We tested more ensembles of eigenfunctions but none can improve over random guessing on the adversarial dataset for small sample sizes. This renders any uniform convergence-based generalization bound still meaningless as it is lower-bounded by $0.5$, which corresponds to random guessing for a binary task.
\begin{figure}
    \centering
    \includegraphics[width=0.3\textwidth]{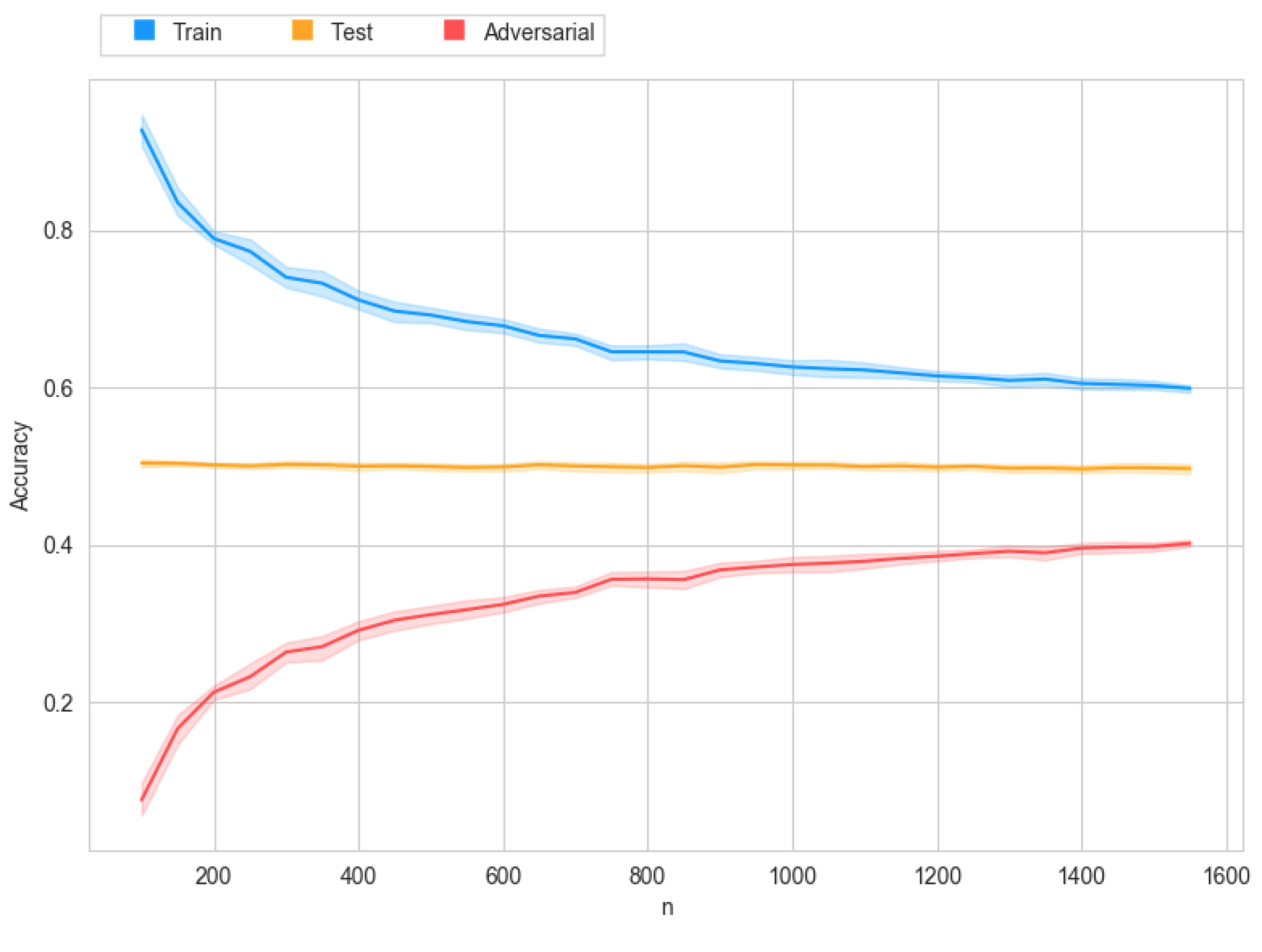}%
    \includegraphics[width=0.3\textwidth]{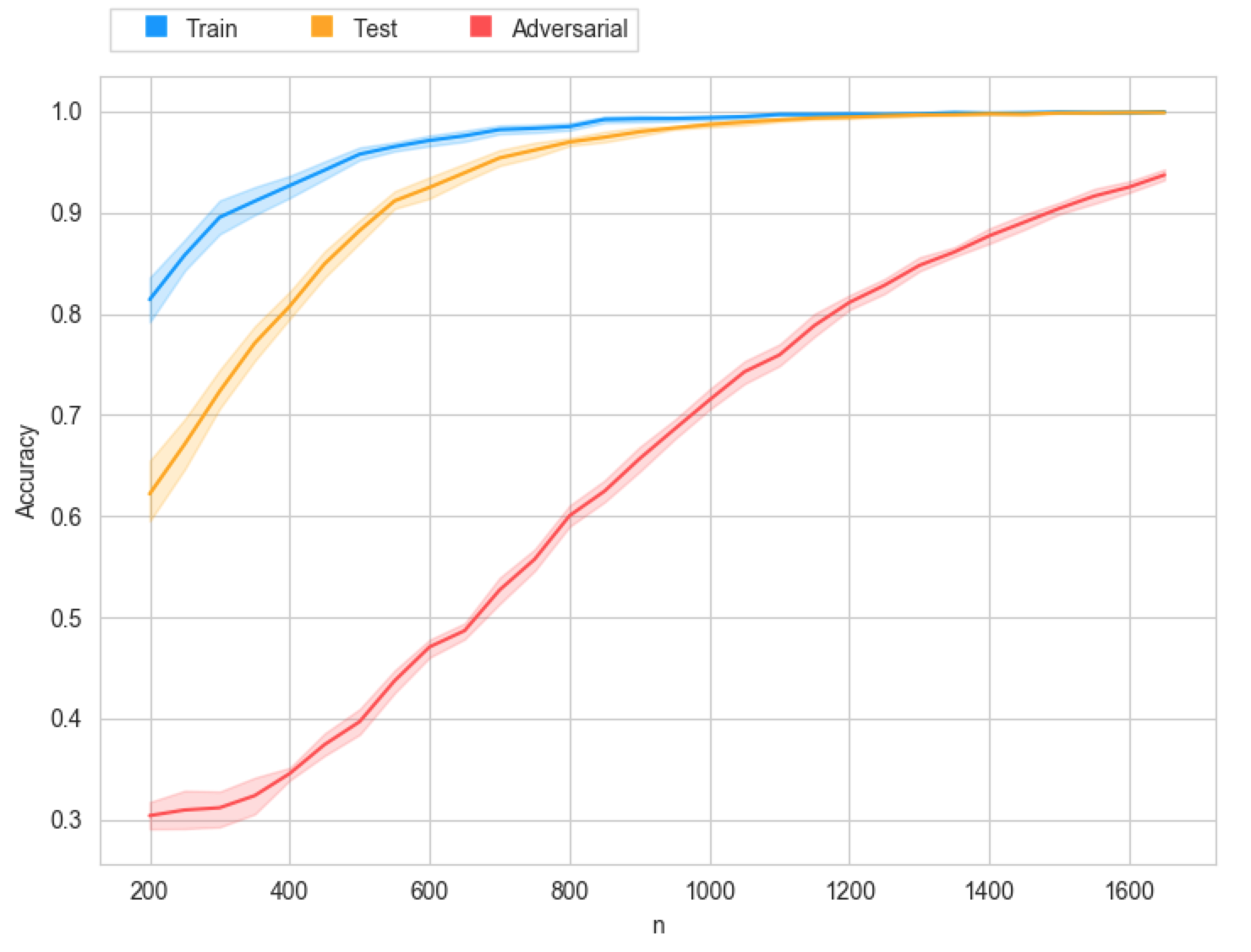}
    \caption{Train, test and adversarial accuracy for the ensemble of eigenfunctions consisting of all but the dominant one (left) and for the ensemble of the $10$ most dominant eigenfunctions}
    \label{decomp}
\end{figure}
\end{appendices}
\end{document}